\documentclass{article} % For LaTeX2e
\usepackage{iclr2027_conference,times}

\usepackage{amsmath,amsfonts,bm}

\def\1{\bm{1}}

\DeclareMathAlphabet{\mathsfit}{\encodingdefault}{\sfdefault}{m}{sl}
\SetMathAlphabet{\mathsfit}{bold}{\encodingdefault}{\sfdefault}{bx}{n}

\usepackage{multirow}
\usepackage{hyperref}
\usepackage{url}
\usepackage{booktabs}
\usepackage[table]{xcolor}
\usepackage{threeparttable}
\usepackage{adjustbox}
\usepackage{algorithm}
\usepackage{tabularx}
\usepackage{makecell}
\usepackage{enumitem}
\usepackage{algpseudocode}
\usepackage{subcaption}
\usepackage[most]{tcolorbox}
\usepackage{xcolor}
\usepackage{amsmath, amssymb, amsthm}
\usepackage{graphicx}

\definecolor{ourblue}{RGB}{205,235,139}

\definecolor{darkgreen}{RGB}{0,130,60}

\hypersetup{
  colorlinks=true,
  citecolor=darkgreen, % 引用保持绿色
  urlcolor=black      % URL / DOI（参考文献里的粉色就是它）
}

\newtheorem{proposition}{Proposition}

\title{Reward-Aligned Reweighting for On-Policy Distillation}

\author{
Haofeng Xu\textsuperscript{1,2}\thanks{Contact: \texttt{haofengxu@connect.hku.hk}.},
Junwei Su\textsuperscript{1,3},
Lansong Diao\textsuperscript{2},
Wenchao Zhou\textsuperscript{2},
Chuan Wu\textsuperscript{1}\thanks{Corresponding author: \texttt{cwu@cs.hku.hk}.}
\\[2mm]
\textsuperscript{1}The University of Hong Kong
\quad
\textsuperscript{2}Alibaba Group
\\
\textsuperscript{3}University of Science and Technology of China
}

\makeatletter
\@ifpackageloaded{amsmath}{}{\usepackage{amsmath}}
\@ifpackageloaded{amssymb}{}{\usepackage{amssymb}}
\@ifpackageloaded{amsthm}{}{\usepackage{amsthm}}
\@ifpackageloaded{natbib}{}{\usepackage{natbib}}
\@ifundefined{proposition}{\newtheorem{proposition}{Proposition}}{}
\makeatother

\input{layout_setup}

\usepackage{wrapfig}
\iclrfinalcopy % Uncomment for camera-ready version, but NOT for submission.
\begin{document}

\maketitle
\lhead{Preprint.} 

\begin{abstract}
On-policy distillation (OPD) has become a standard stage of large language model post-training: a student learns from a stronger teacher on trajectories generated by its own policy. Yet on-policy feedback does not make every teacher correction equally useful. Standard OPD weights token-level distillation terms uniformly, implicitly treating local teacher preference as a proxy for correction utility. A decision's task value, however, depends on how the student completes the subsequent reasoning. This mismatch can cause imitation to suppress viable student strategies or reinforce paths the student cannot reliably execute.
Verified trajectory outcomes provide complementary evidence about continuation quality, but do not directly identify the utility of individual decisions. This motivates using outcome evidence to guide the relative strength of teacher corrections.
We introduce Reward-Aligned Reweighting for On-Policy Distillation (R$^{2}$-OPD), which continuously reallocates teacher supervision using outcome agreement and the magnitude of teacher--student disagreement. It gives reward-aligned corrections greater relative influence while retaining dense feedback, moving beyond uniform imitation and hard filtering.
Our analysis formalizes the mismatch between local teacher preference and student continuation value, and establishes sufficient conditions for reallocation to improve first-order task progress over uniform OPD. 
Across seven mathematical reasoning benchmarks, R$^{2}$-OPD achieves the highest average accuracy among the compared training methods in both cross-size and same-size distillation. It outperforms standard OPD on all seven benchmarks, with average gains of $+3.5$ and $+2.4$ for 1.7B and 4B students, respectively. An extension to code generation yields an average gain of $+1.6$ over standard OPD. Together, these results highlight outcome-guided supervision allocation as an effective means of translating dense teacher feedback into stronger student performance across model scales and task domains. 
% Code is available at \url{https://anonymous.4open.science/r/r2-opd-62F4/}.
\end{abstract}

\section{Introduction}
\label{sec:introduction}

On-policy distillation (OPD) trains a student language model with dense teacher feedback on trajectories generated by the student itself. By querying the teacher at student-visited prefixes, it reduces the mismatch between training contexts and inference-time behavior \citep{agarwal2024gkd,gu2024minillm,lu2025onpolicy}. However, aligning the contexts in which supervision is provided does not ensure that every teacher correction helps the student complete the task. Effective distillation therefore requires considering not only \emph{where} teacher feedback is obtained, but also \emph{how strongly} each correction should influence learning.

Standard sampled-token OPD uses teacher--student log-probability gaps as token-level correction coefficients, without outcome-dependent reweighting. These coefficients encode \emph{local teacher preference}, whereas task success depends on \emph{student continuation value}: the expected reward after a decision when the student, rather than the teacher, generates the remaining response. This distinction gives rise to two potential failure modes. \textbf{Over-correction on successful trajectories:} suppressing teacher-disfavored decisions can erode valid student strategies that differ from the teacher's \citep{yuan2023scaling}. \textbf{Over-trust on failed trajectories:} reinforcing teacher-preferred decisions can encourage reasoning paths that the student cannot reliably complete. In Section~\ref{sec:motivation}, we formalize this preference--value mismatch: moving a local next-token distribution toward the teacher can reduce reverse KL while decreasing expected task reward when the student's continuation policy is held fixed. Closer agreement with the teacher is therefore not, by itself, sufficient evidence of useful supervision.

Verified trajectory outcomes provide complementary evidence about continuation quality, but do not identify the utility of individual decisions. Existing approaches use outcome feedback to combine learning objectives or select teacher supervision \citep{li2026sequential,xu2026sgopd,lin2026opdvr}. However, deciding whether to retain a correction does not resolve how strongly it should influence learning, and hard filtering discards feedback whose utility remains uncertain. This motivates a different use of outcomes: \emph{guiding correction strength rather than treating outcome agreement as a definitive token-value label}. 

We introduce Reward-Aligned Reweighting for On-Policy Distillation (\textbf{R$^{2}$-OPD}), a continuous reweighting method that combines outcome agreement with the magnitude of teacher--student disagreement. It gives reward-aligned corrections greater relative influence while retaining feedback that disagrees with the outcome. The reweighting preserves correction signs and each response's total absolute coefficient mass, reallocating rather than simply reducing teacher supervision. Our analysis establishes sufficient reward--utility agreement conditions under which this reallocation improves first-order task progress relative to uniform OPD.  R$^{2}$-OPD uses outcomes to guide imitation without requiring a separate reinforcement-learning objective or a learned value estimator.

We evaluate R$^{2}$-OPD in same-size and cross-size distillation settings studied in prior OPD work~\citep{lin2026opdvr,yang2026gopd}. Across seven mathematical reasoning benchmarks, R$^{2}$-OPD achieves the highest average accuracy among the compared methods in both settings, with gains of $+3.5$ and $+2.4$ over standard OPD for 1.7B and 4B students, respectively. An extension to code generation yields an average gain of $+1.6$ over standard OPD. Learning curves show early performance gains, while ablations support the joint use of outcome alignment, disagreement magnitude, and coefficient-mass normalization.

\paragraph{Contributions.}
\textbf{(1) Preference--value analysis.}
We formalize the mismatch between local teacher preference and student continuation value, and establish sufficient conditions for supervision reallocation to improve first-order task progress over standard OPD.
\textbf{(2) Reward-aligned supervision allocation.}
We introduce R$^{2}$-OPD, a continuous, magnitude-sensitive reweighting method that uses verified outcomes while preserving nonzero teacher feedback, correction signs, and per-response absolute coefficient mass.
\textbf{(3) Empirical validation across tasks and scales.}
Experiments on seven mathematical reasoning benchmarks and two code-generation benchmarks demonstrate gains over standard OPD, with ablations examining the roles of the key reweighting components.

\section{Preliminaries}
\label{sec:preliminaries}

\subsection{Notation and Distillation Objective}
\label{subsec:kl}
Let $q\sim\mathcal D$ denote a prompt and $o=(o_1,\ldots,o_{|o|})$ a response, with $s_t=(q,o_{<t})$ and $\pi(o\mid q)=\prod_{t=1}^{|o|}\pi(o_t\mid s_t)$. We consider a trainable student $\pi_\theta$ and a frozen teacher $\pi_T$ over a shared finite vocabulary $\mathcal V$. At each iteration, rollouts are generated by a fixed student snapshot $\pi_{\bar\theta}$. For the identities below, we assume a differentiable student, positive teacher and student probabilities, and a bounded generation horizon. Following reverse-KL distillation~\citep{gu2024minillm}, we define
\begin{equation}
\begin{aligned}
K_t(\theta)
&=D_{\mathrm{KL}}\!\left(
\pi_\theta(\cdot\mid s_t)\,\|\,\pi_T(\cdot\mid s_t)
\right)\\
&=\sum_{v\in\mathcal V}\pi_\theta(v\mid s_t)
\log\frac{\pi_\theta(v\mid s_t)}{\pi_T(v\mid s_t)}.
\end{aligned}
\label{eq:state_rkl}
\end{equation}

\subsection{On-Policy Distillation}
\label{subsec:opd}
OPD supervises the student on prefixes visited by its own rollouts~\citep{agarwal2024gkd}. Given $o\sim\pi_{\bar\theta}(\cdot\mid q)$, the statewise reverse-KL objective is
\begin{equation}
\mathcal L_{\mathrm{OPD}}(\theta;\bar\theta)
=
\mathbb E_{\substack{
q\sim\mathcal D\\
o\sim\pi_{\bar\theta}(\cdot\mid q)
}}
\left[
\sum_{t=1}^{|o|}K_t(\theta)
\right].
\label{eq:opd_objective}
\end{equation}
The rollout distribution and sampled prefixes are held fixed when differentiating with respect to $\theta$. Evaluating the full-vocabulary KL requires a vocabulary-wide comparison at every visited prefix, which can be expensive for large vocabularies and long responses. We therefore focus on sampled-token OPD, which uses teacher log-probabilities at student-sampled tokens~\citep{lu2025onpolicy}.
For $o_t\sim\pi_{\bar\theta}(\cdot\mid s_t)$, define the
\emph{distillation advantage}
\begin{equation}
A_t
=
\log\frac{\pi_T(o_t\mid s_t)}
{\pi_{\bar\theta}(o_t\mid s_t)}.
\label{eq:sampled_advantage}
\end{equation}
Conditional on $s_t$, this gives $\mathbb E[-A_t\mid s_t]=K_t(\bar\theta)$, so $-A_t$ is an unbiased single-sample estimate of the local reverse KL at the rollout parameters. For analysis, we use the detached log-probability surrogate
\begin{equation}
\mathcal L_{\mathrm{OPD}}^{\mathrm{sample}}(\theta;\bar\theta)
=
-\mathbb E_{\substack{
q\sim\mathcal D\\
o\sim\pi_{\bar\theta}(\cdot\mid q)
}}
\left[
\sum_{t=1}^{|o|}
\operatorname{sg}[A_t]\log\pi_\theta(o_t\mid s_t)
\right],
\label{eq:opd_surrogate}
\end{equation}
where $\operatorname{sg}[\cdot]$ denotes the stop-gradient
operator. At $\theta=\bar\theta$, its expected gradient matches that of
the fixed-rollout reverse-KL objective:
\begin{equation}
\left.
\nabla_\theta
\mathcal L_{\mathrm{OPD}}^{\mathrm{sample}}(\theta;\bar\theta)
\right|_{\theta=\bar\theta}
=
\left.
\nabla_\theta\mathcal L_{\mathrm{OPD}}(\theta;\bar\theta)
\right|_{\theta=\bar\theta}.
\label{eq:opd_gradient_consistency}
\end{equation}
This is a first-order identity at the rollout parameters, not an equality of loss values or a derivative through the state-visitation distribution. In an individual surrogate term, positive and negative $A_t$ contribute to reinforcing and suppressing the sampled token, respectively, while $|A_t|$ measures the sampled-token log-probability discrepancy. Standard sampled-token OPD uses these coefficients without additional outcome-dependent reweighting: they encode local teacher preference rather than verified trajectory outcomes.

\subsection{Reinforcement Learning with Verifiable Rewards}
\label{subsec:rlvr}
We consider tasks with a binary verifier $R(q,o)\in\{0,1\}$. Reinforcement learning with verifiable rewards (RLVR) seeks to maximize the expected task reward
\begin{equation}
J(\theta)
=
\mathbb E_{\substack{
q\sim\mathcal D\\
o\sim\pi_\theta(\cdot\mid q)
}}
\left[R(q,o)\right].
\label{eq:rlvr_objective}
\end{equation}
Reinforcement learning with verifiable rewards provides a trajectory-level signal indicating whether a sampled response satisfies the task verifier. We encode this outcome as
\begin{equation}
z(q,o)=2R(q,o)-1\in\{-1,+1\},
\label{eq:signed_outcome}
\end{equation}
where $+1$ denotes success and $-1$ denotes failure. The outcome is shared by all tokens in the response and does not identify the utility of any individual decision. This differs from the token-specific distillation signal $A_t$, which measures the teacher's preference for the sampled token at position $t$.
\section{Teacher Preference versus Student Continuation Value}
\label{sec:motivation}

The over-correction and over-trust described in the Introduction raise a question: when does local teacher imitation improve the student's expected task reward? We isolate this question through a local distributional intervention using the distillation advantage $A_t$ defined in Eq.~\eqref{eq:sampled_advantage}.

\paragraph{Local imitation versus student value.}
Fix a visited prefix $s=s_t$. The task value of the next action depends on how the student completes the response:
\begin{equation}
Q_R^{\bar\theta}(s,a)
=
\mathbb E_{o_{>t}\sim\pi_{\bar\theta}(\cdot\mid s,a)}
\left[R(q,o_{<t}a o_{>t})\right].
\label{eq:mot_student_value}
\end{equation}
The continuation follows the rollout student, not the teacher. To isolate the effect of local imitation, we interpolate only the next-action distribution toward the teacher:
\begin{equation}
\pi_\epsilon(a\mid s)
\propto
\pi_{\bar\theta}(a\mid s)^{1-\epsilon}
\pi_T(a\mid s)^\epsilon,
\qquad 0\leq\epsilon\leq1.
\label{eq:mot_local_interpolation}
\end{equation}
The distribution is normalized over actions, while all subsequent decisions remain governed by $\pi_{\bar\theta}$. The continuation values $Q_R^{\bar\theta}(s,a)$ remain fixed along the interpolation.
Let $D_s(\epsilon)$ denote the reverse KL in Eq.~\eqref{eq:state_rkl} evaluated at $\pi_\epsilon(\cdot\mid s)$, and let $F_s(\epsilon)=\mathbb E_{a\sim\pi_\epsilon(\cdot\mid s)} [Q_R^{\bar\theta}(s,a)]$ denote the corresponding conditional expected reward. Moments with subscript $s$ are taken over $o_t\sim\pi_{\bar\theta}(\cdot\mid s)$.

\begin{proposition}[Local imitation need not improve student value]
\label{prop:mot_local_separation}
For a finite action set with positive teacher and student probabilities, the interpolation in Eq.~\eqref{eq:mot_local_interpolation} satisfies
\begin{align}
D_s'(0)
&=-\operatorname{Var}_s[A_t],
\label{eq:mot_kl_derivative}\\
F_s'(0)
&=\operatorname{Cov}_s\!\left(
A_t,Q_R^{\bar\theta}(s,o_t)\right).
\label{eq:mot_value_derivative}
\end{align}
If the covariance is negative, every sufficiently small positive $\epsilon$ strictly decreases both the local reverse KL and the student's conditional expected reward.
\end{proposition}

The proof is given in Appendix~\ref{app:motivation-proofs}. The proposition separates distributional agreement from task utility: local reverse KL decreases to first order whenever $A_t$ is nonconstant, whereas the reward change depends on how teacher preference covaries with student continuation value. To relate this distinction to the two failure modes, write
\begin{equation}
F_s'(0)=\mathbb E_s\!\bigl[(A_t-\mathbb E_s[A_t])
(Q_R^{\bar\theta}(s,o_t)-F_s(0))\bigr].
\label{eq:mot_value_balance}
\end{equation}
Under this interpolation, the centered preference $A_t-\mathbb E_s[A_t]$ determines the direction of first-order probability change. The value-based forms of over-correction and over-trust are therefore shifts away from actions with above-average student continuation value and toward those with below-average value, respectively, with $F_s(0)$ as the reference. Both contribute negatively to Eq.~\eqref{eq:mot_value_balance}, although contributions from other actions can offset them. This is a local probability-space analysis, not a characterization of an arbitrary parameter-space update.

\paragraph{Implications for supervision allocation.}
Trajectory outcomes provide evidence about student continuation value: under student-generated continuations, $\mathbb E[z(q,o)\mid s_t=s,o_t=a] =2Q_R^{\bar\theta}(s,a)-1$. Although disagreement $z(q,o)A_t<0$ does not establish that an individual correction is harmful, outcome feedback complements local teacher preference with evidence about the student's ability to complete the task. This complementarity motivates R$^{2}$-OPD, which uses reward alignment to guide the relative strength of teacher corrections.
\section{Reward-Aligned Reweighting}
\label{sec:method}

To realize this allocation principle, we combine continuous token-level gating with per-response normalization. The gate adjusts correction strength using reward alignment and disagreement magnitude, while normalization preserves the original absolute coefficient mass within each response.

\subsection{Reward-Aligned Distillation Advantages}
\label{subsec:r2_reweighting}

For a complete response $o$, let $A_t$ and $z(q,o)$ denote the distillation advantage and signed outcome defined in Section~\ref{sec:preliminaries}. A correction is \emph{reward-aligned} if $z(q,o)A_t>0$ and \emph{reward-inconsistent} if $z(q,o)A_t<0$. We combine alignment and disagreement magnitude through a response-local scale and a sigmoid gate:
\begin{equation}
    \nu_o=\max\!\left\{
        \operatorname{median}_{1\leq j\leq |o|}|A_j|,\epsilon
    \right\},\qquad \epsilon>0,
\label{eq:r2_scale}
\end{equation}
\begin{equation}
    g_t=\sigma\!\left(\beta\frac{z(q,o)A_t}{\nu_o}\right),
    \qquad \beta\geq0,
\label{eq:r2_score_gate}
\end{equation}
where $\sigma(u)=(1+e^{-u})^{-1}$. Scaling by $\nu_o$ expresses disagreement relative to its typical within-response magnitude, while $\beta$ controls the sharpness of the weighting. For $\beta>0$, aligned corrections receive $g_t>1/2$ and inconsistent corrections receive $g_t<1/2$. When $\beta|A_t|/\nu_o$ is small, the gate remains close to $1/2$, producing a gradual adjustment. For finite $\beta$, the gate remains strictly positive.

Since gating alone reduces absolute coefficient mass, we normalize each response to preserve $M_o=\sum_t|A_t|$. For $M_o>0$, the \emph{reweighted distillation advantage} is:
\begin{equation}
    Z_o=\frac{M_o}{\sum_j g_j|A_j|},
    \qquad A_t^{\mathrm{R2}}=Z_og_tA_t.
\label{eq:r2_mass_weight}
\end{equation}
The common factor $Z_o$ restores the original coefficient budget without changing the relative gate multipliers across positions.

\subsection{Training Objective}
\label{subsec:r2_update}

Substituting $A_t^{\mathrm{R2}}$ into Eq.~\eqref{eq:opd_surrogate}
gives the token-summed analysis surrogate
\begin{equation}
    \mathcal L_{\mathrm{R2}}(\theta;\bar\theta)
    =-\mathbb E\!\left[
        \sum_t\operatorname{sg}[A_t^{\mathrm{R2}}]\,
        \log\pi_\theta(o_t\mid s_t)
    \right].
\label{eq:r2_loss}
\end{equation}
Advantages and reweighting factors are detached; gradients flow only
through $\log\pi_\theta(o_t\mid s_t)$.
Reward modulates existing token-level contributions without an additional
policy-gradient term.
Algorithm~\ref{alg:r2opd} summarizes reweighting; the implemented loss,
response-wise reduction, and gradient correspondence are detailed in
Appendix~\ref{app:optimization_protocol}.

\begin{algorithm}[t]
\caption{One step of R$^{2}$-OPD.}
\label{alg:r2opd}
\small
\algrenewcommand{\algorithmicindent}{1em}
\algrenewcommand{\algorithmiccomment}[1]{%
    \hfill{\footnotesize\itshape $\triangleright$~#1}}
\algtext*{EndFor}
\algtext*{EndIf}

\begin{algorithmic}[1]
\Require Student $\pi_\theta$, teacher $\pi_T$, prompts $\mathcal D$,
         verifier $R$; sharpness $\beta\geq0$, scale floor $\epsilon>0$.
\State Freeze $\bar\theta\gets\theta$; sample
       $\{q_i\}_{i=1}^{B}\sim\mathcal D$.
\State Sample $o_i\sim\pi_{\bar\theta}(\cdot\mid q_i)$;
       collect $\mathcal B=\{(q_i,o_i)\}_{i=1}^{B}$.
\For{each $(q,o)\in\mathcal B$ with $|o|>0$}
    \State Query $\pi_T$ to compute $A_t$; obtain $z(q,o)$ from $R$.
           \Comment{token-wise teacher feedback}
    \State $A_t^{\mathrm{R2}}\gets A_t$;
           $M_o\gets\sum_t|A_t|$.
           \Comment{vanilla fallback}
    \If{$o$ is complete and $M_o>0$}
        \State Compute $\nu_o$;
               $g_t\gets\sigma(\beta z(q,o)A_t/\nu_o)$.
               \Comment{soft alignment}
        \State $Z_o\gets M_o/(\sum_j g_j|A_j|)$;
               $A_t^{\mathrm{R2}}\gets Z_og_tA_t$.
               \Comment{mass-preserving reweighting}
    \EndIf
\EndFor
\State Update $\theta$ via the OPD surrogate with $\operatorname{sg} [A_t^{\mathrm{R2}}]$.
\end{algorithmic}
\end{algorithm}

\subsection{Structural Properties}
\label{subsec:r2_guarantee}

The multiplicative reweighting admits the following exact decomposition into a teacher-preference component and an outcome-dependent modulation.

\begin{proposition}[Exact teacher--outcome decomposition]
\label{prop:r2_decomposition}
For a complete response with finite advantages, $M_o>0$, and finite $\beta\geq0$,
\begin{equation}
 A_t^{\mathrm{R2}}
 =\underbrace{\frac{Z_o}{2}A_t}_{\substack{\text{Term A}\\
                                      \text{teacher preference}}}
 +\underbrace{\frac{Z_o}{2}z(q,o)|A_t|
    \tanh\!\left(\frac{\beta|A_t|}{2\nu_o}\right)}_{
       \substack{\text{Term B}\\ \text{reward-aligned modulation}}}.
\label{eq:r2_decomposition}
\end{equation}
\end{proposition}
Term A preserves the teacher's signed correction under a common response-level scale. Term B is nonnegative on successful responses and nonpositive on failed responses. Relative to Term A, it attenuates suppression on successful responses and reinforcement on failed responses, while strengthening reward-aligned corrections. This provides a coefficient-level interpretation of how the reweighting addresses over-correction and over-trust. The $\tanh$ factor controls the modulation strength through the normalized disagreement.

For nonzero $A_t$ and finite $\beta$, Term B has strictly smaller magnitude than Term A, so their sum preserves the original coefficient sign. The two terms form a decomposition of a single distillation coefficient. Term B is an outcome-dependent component of this coefficient, rather than a separate or generally unbiased task-advantage estimator. Appendix~\ref{app:r2_decomposition} provides the proof.

We quantify supervision allocation by coefficient mass. For $M_o>0$, define the \emph{reward-inconsistent mass share}
\begin{equation}
    \kappa_o(\beta)=
    \frac{\sum_{t:z(q,o)A_t<0}|A_t^{\mathrm{R2}}|}
         {\sum_t|A_t^{\mathrm{R2}}|}.
\label{eq:r2_diagnostics}
\end{equation}
\begin{proposition}[Mass-preserving reward-aligned allocation]
\label{prop:r2_path}
Fix a complete response with finite advantages and $M_o>0$. For every finite $\beta\geq0$, reweighting retains all nonzero corrections and their signs, with
\begin{equation}
 \sum_t|A_t^{\mathrm{R2}}|=M_o,
 \qquad \kappa_o'(\beta)<0\ \text{if }0<\kappa_o(0)<1.
\label{eq:r2_allocation_guarantee}
\end{equation}
The derivative at zero is right-sided. At $\beta=0$, $A_t^{\mathrm{R2}}=A_t$. If aligned mass is positive, the limit $\beta\to\infty$ is hard reward-aligned filtering renormalized to mass $M_o$.
\end{proposition}

When both groups carry mass, increasing $\beta$ monotonically shifts a fixed coefficient budget toward reward-aligned feedback. This does not imply monotonic changes for individual token coefficients. As $\beta\to\infty$, the update converges to mass-normalized hard reward-aligned gating. Preserving coefficient mass does not preserve the parameter-gradient norm, since token score vectors may reinforce or cancel one another. The proof and boundary cases are given in Appendix~\ref{app:r2_path}.

\subsection{When Reallocation Improves Task Progress}
\label{subsec:r2_task_progress}

The preceding results describe supervision allocation, but not its task utility. We therefore connect the allocation to first-order progress on the expected reward
$J(\theta)=\mathbb E_{q\sim\mathcal D,\,o\sim\pi_\theta}[R(q,o)]$.
Let $g_R=\nabla J(\bar\theta)$ and
$\psi_t=\left.\nabla_\theta\log\pi_\theta(o_t\mid s_t)
\right|_{\theta=\bar\theta}$.
For a fixed complete response with $M_o>0$, define
\begin{equation}
p_o(t)=\frac{|A_t|}{M_o},
\qquad
u_t=\operatorname{sgn}(A_t)\langle g_R,\psi_t\rangle .
\label{eq:r2_task_utility}
\end{equation}
Here, $p_o(t)$ is the original absolute-coefficient mass share at position $t$, not an action probability, and $u_t$ is the task-gradient projection per unit mass in the teacher-directed score direction. Its sign need not coincide with reward alignment. The corresponding per-response surrogate directions are
$d_{\mathrm V}(o)=\sum_t A_t\psi_t$ and
$d_{\mathrm{R2}}(o)=\sum_t A_t^{\mathrm{R2}}\psi_t$.

\begin{proposition}[Conditional task-progress identity]
\label{prop:r2_task_progress}
For the fixed response and directions above,
\begin{equation}
 \langle g_R,d_{\mathrm{R2}}(o)-d_{\mathrm V}(o)\rangle
 =\frac{M_o\operatorname{Cov}_{p_o}(g,u)}
 {\mathbb E_{p_o}[g]}.
\label{eq:r2_task_gain}
\end{equation}
\end{proposition}

Since $\mathbb E_{p_o}[g]>0$, reweighting improves the response-level first-order task projection over standard OPD if and only if $\operatorname{Cov}_{p_o}(g,u)>0$. Thus, larger gates must be associated with higher task-projection utility under the original mass distribution. Reducing reward-inconsistent mass alone does not guarantee task progress. Neither $g_R$ nor $u_t$ is a training input; both are used only for analysis. Appendix~\ref{app:r2_task_progress} proves the identity and gives a sufficient reward--utility agreement condition. Appendices~\ref{app:r2_finite_step} and \ref{app:r2_compatibility} distinguish this local criterion from finite-step reward ascent and surrogate-gradient compatibility.
\section{Experiments}
\label{sec:experiments}

We assess R$^{2}$-OPD across student sizes and task domains, with controlled ablations to examine the role of reward-aligned supervision allocation.

\subsection{Experimental Settings}
\label{subsec:exp_setup}

\paragraph{Models and training data.}
We initialize students from Qwen3-4B-Base and Qwen3-1.7B-Base~\citep{yang2025qwen3}. The teacher, denoted Qwen3-4B-GRPO, is obtained by applying GRPO~\citep{shao2024deepseekmath} to Qwen3-4B-Non-Thinking. The same frozen checkpoint is shared across all mathematical distillation methods and student sizes. Both teacher RL and student OPD use DAPO-Math-17k~\citep{yu2025dapo}. We use a fixed $\beta=0.001$ across all R$^{2}$-OPD experiments, without experiment-specific tuning. Full training details are provided in Appendix~\ref{app:exp_reproducibility}.

\paragraph{Evaluation.}
We report $\mathrm{avg}@16$ on AIME24/25~\citep{maaaime}, and $\mathrm{avg}@4$ on AMC23/24~\citep{maaamc}, MATH500~\citep{hendrycks2021math,lightman2023lets}, MinervaMath~\citep{lewkowycz2022minerva}, and OlympiadBench~\citep{he2024olympiadbench}. For the main 1.7B cross-size mathematical comparison, each method uses three independent training seeds with the same initial student and fixed teacher checkpoints, and results are reported as mean $\pm$ sample standard deviation. Given the computational budget, all remaining experiments use one training run per method or variant. The two initial students and fixed teacher are each evaluated once under the same benchmark-specific sampling protocol. Evaluation details are provided in Appendix~\ref{app:exp_evaluation}.

\paragraph{Baselines.}
We compare sampled-token OPD~\citep{lu2025onpolicy}, OPDVR~\citep{lin2026opdvr}, SG-OPD~\citep{xu2026sgopd}, and ExOPD~\citep{yang2026gopd}. We also include Top-64 OPD, our top-$k$ implementation control of reverse-KL distillation~\citep{gu2024minillm}, and OPD + GRPO, an additive GRPO--OPD control~\citep{shao2024deepseekmath,lu2025onpolicy}. The initial students and fixed teacher are reference models, not additional trained baselines.

\begin{figure}[!t]
    \centering
    \includegraphics[width=\linewidth]
        {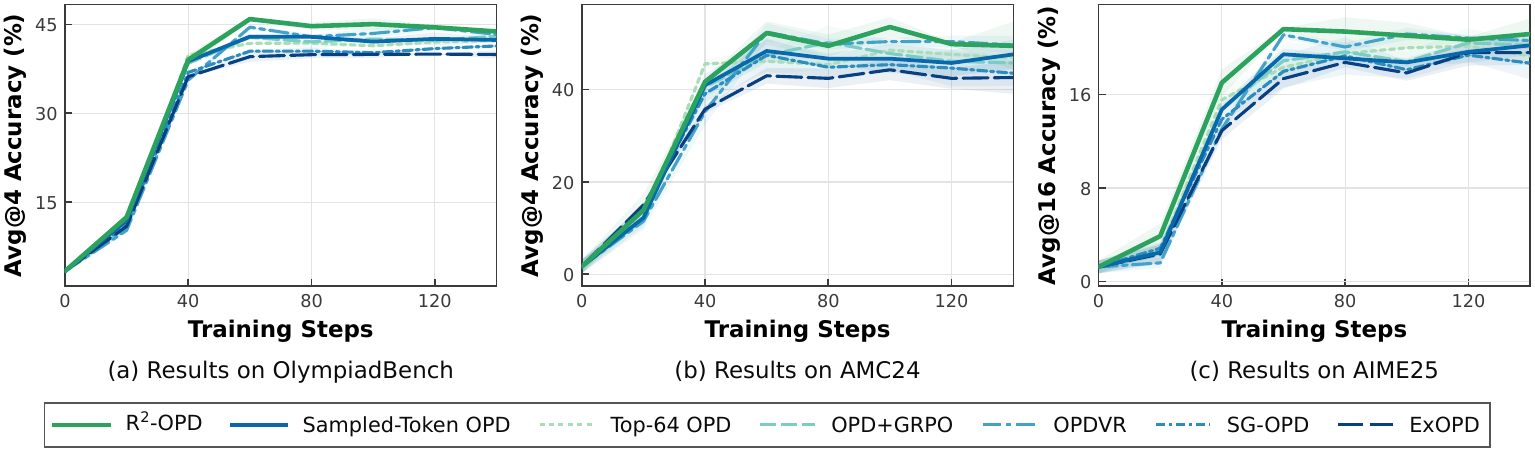}
    \caption{
        Evaluation accuracy on OlympiadBench, AMC24, and AIME25
        during cross-size distillation from Qwen3-4B-GRPO to
        Qwen3-1.7B-Base.
        Curves show the mean, and shaded bands indicate
        $\pm$ one sample standard deviation across three
        independent training seeds.
    }
    \label{fig:main_results_accuracy_17b}
\end{figure}

\begin{table*}[!t]
\centering
\caption{Accuracy (\%) on mathematical reasoning benchmarks under
cross-size and same-size distillation.
Within each distillation setting, \textbf{bold} and
\underline{underlined} scores indicate the best and second-best
results in each column, respectively.}
\label{tab:main_results}

\small
\setlength{\tabcolsep}{2.5pt}
\setlength{\arrayrulewidth}{0.35pt}
\renewcommand{\arraystretch}{1.05}

\begin{adjustbox}{max width=\textwidth}
\begin{tabular}{@{}l|*{7}{c}|c@{}}
\toprule
Method
& AIME24
& AIME25
& AMC23
& AMC24
& MATH500
& Minerva
& Olympiad
& Avg. \\
\midrule

% ================= Reference models =================
\multicolumn{9}{c}{
    \textit{Reference models}
} \\
\midrule

Qwen3-1.7B-Base\,{\scriptsize (Student)}
& 1.9 & 1.7 & 4.4 & 2.2
& 5.8 & 3.1 & 3.7 & 3.3 \\

Qwen3-4B-Base\,{\scriptsize (Student)}
& 3.1 & 1.4 & 11.9 & 3.3
& 8.4 & 4.4 & 4.8 & 5.3 \\

Qwen3-4B-GRPO\,{\scriptsize (Teacher)}
& 30.1 & 26.5 & 62.5 & 51.1
& 77.2 & 31.0 & 46.1 & 46.4 \\
% ============ Cross-size distillation ============
\midrule
\multicolumn{9}{c}{
    \textit{Qwen3-4B-GRPO $\rightarrow$ Qwen3-1.7B-Base}
} \\
\midrule

Sampled-Token OPD
& $25.5_{\pm 1.1}$ & $20.2_{\pm 0.8}$ & $60.0_{\pm 2.7}$
& $48.3_{\pm 2.4}$ & $74.7_{\pm 0.2}$ & $27.6_{\pm 0.9}$
& $42.9_{\pm 0.5}$ & $42.7_{\pm 0.6}$ \\

Top-64 OPD
& $25.0_{\pm 1.3}$ & $20.1_{\pm 0.9}$ & $60.6_{\pm 1.0}$
& $48.5_{\pm 2.7}$ & $74.5_{\pm 1.2}$ & $27.1_{\pm 0.4}$
& $42.3_{\pm 0.0}$ & $42.6_{\pm 0.4}$ \\

OPD + GRPO
& $26.1_{\pm 1.0}$ & \underline{$21.2_{\pm 1.2}$}
& $60.6_{\pm 3.4}$ & $50.2_{\pm 3.6}$ & $75.1_{\pm 0.3}$
& $\bm{30.6_{\pm 0.6}}$ & $43.0_{\pm 0.4}$ & $43.8_{\pm 0.9}$ \\

OPDVR
& $\bm{29.8_{\pm 1.0}}$ & $20.4_{\pm 0.5}$
& $\bm{63.3_{\pm 3.0}}$ & \underline{$52.4_{\pm 2.0}$}
& \underline{$76.3_{\pm 1.3}$} & $28.0_{\pm 1.0}$
& \underline{$44.5_{\pm 0.3}$} & \underline{$45.0_{\pm 0.3}$} \\

SG-OPD
& $23.6_{\pm 0.2}$ & $19.4_{\pm 0.6}$ & $58.4_{\pm 2.4}$
& $47.4_{\pm 3.6}$ & $73.7_{\pm 0.2}$ & $27.8_{\pm 0.9}$
& $41.4_{\pm 0.3}$ & $41.7_{\pm 0.2}$ \\

ExOPD
& $24.9_{\pm 1.1}$ & $19.6_{\pm 1.1}$ & $59.0_{\pm 3.4}$
& $44.3_{\pm 2.2}$ & $71.8_{\pm 0.2}$ & $27.9_{\pm 0.1}$
& $40.0_{\pm 0.0}$ & $41.1_{\pm 0.1}$ \\

\rowcolor{ourblue}
\textbf{R$^{2}$-OPD}
& \underline{$29.4_{\pm 0.8}$}
& $\bm{24.6_{\pm 0.9}}$
& \underline{$63.0_{\pm 2.0}$}
& $\bm{53.5_{\pm 2.1}}$
& $\bm{77.6_{\pm 0.8}}$
& \underline{$29.1_{\pm 0.6}$}
& $\bm{45.9_{\pm 0.3}}$
& $\bm{46.2_{\pm 0.4}}$ \\

% ============ Same-size distillation ============
\midrule
\multicolumn{9}{c}{
    \textit{Qwen3-4B-GRPO $\rightarrow$ Qwen3-4B-Base}
} \\
\midrule

Sampled-Token OPD
& 32.9 & 26.8 & 65.5 & 53.8 & 80.6 & 32.1 & 47.3 & 48.4 \\

Top-64 OPD
& 32.6 & 27.0 & 65.2 & 53.2 & 80.1 & 31.6 & 46.9 & 48.1 \\

OPD + GRPO
& 33.3 & 26.5 & 66.5 & \underline{57.3} & 82.3 & 33.1 & 48.1 & 49.6 \\

OPDVR
& 33.5 & 26.9 & 67.6 & 55.2 & \textbf{84.6} & 32.5 & \underline{48.4} & 49.8 \\

SG-OPD
& 33.1 & 26.6 & 67.9 & 54.9 & 81.8 & 31.3 & 47.6 & 49.0 \\

ExOPD
& \underline{34.7} & \underline{27.6} & \textbf{68.9}
& 54.5 & 82.1 & \textbf{33.9} & 47.5 & \underline{49.9} \\

\rowcolor{ourblue}
\textbf{R$^{2}$-OPD}
& \textbf{35.1} & \textbf{28.1} & \underline{68.3}
& \textbf{57.8} & \underline{83.6} & \underline{33.5}
& \textbf{48.9} & \textbf{50.8} \\

\bottomrule
\end{tabular}
\end{adjustbox}
\end{table*}

\subsection{Main Results}
\label{subsec:exp_main_results}

As shown in Table~\ref{tab:main_results}, R$^{2}$-OPD achieves the highest average accuracy in both distillation settings and outperforms standard OPD on all seven benchmarks. It reaches $46.2\pm0.4\%$ with the 1.7B student and $50.8\%$ with the 4B student, improving over standard OPD with average gains of $+3.5$ and $+2.4$, respectively. It also surpasses the strongest competing baselines, OPDVR and ExOPD, by $+1.2$ and $+0.9$, respectively. The gains are especially pronounced on challenging competition-mathematics benchmarks, reaching $+4.4$ on AIME25, $+5.2$ on AMC24, and $+3.0$ on OlympiadBench in the 1.7B setting. 
Figure~\ref{fig:main_results_accuracy_17b} shows that these gains emerge early: a clear mean-accuracy advantage on AIME25 is already visible around step~40. Across all plotted benchmarks, R$^{2}$-OPD approaches near-peak accuracy by roughly step~60 and generally maintains higher mean accuracy thereafter, indicating both faster early learning and sustained
performance gains.

% Replace BOTH existing subsections, including their minipage blocks.
% The surrounding document must be at normal text width.
\par
\Needspace{18\baselineskip}
\subsection{Extension to Code Generation}
\label{subsec:exp_coding}

\begingroup
\setlength{\intextsep}{2pt}
\setlength{\columnsep}{1em}
\setlength{\parindent}{0pt}
% True wrapping table. Do not place this file inside a minipage or table.
\begin{wraptable}{r}{0.50\columnwidth}
\vspace{-6pt} % 新增：上移 caption 和表体
\centering
\setlength{\abovecaptionskip}{0pt}
\setlength{\belowcaptionskip}{4pt}
\caption{Code-generation $\mathrm{avg}@4$ (\%) under cross-size
 distillation from Qwen3-4B-GRPO to Qwen3-1.7B-Base.
 \textbf{Bold} and \underline{underlined} scores indicate the best
 and second-best results among student-training methods.}
\label{tab:coding_results_17b}
\scriptsize
\setlength{\tabcolsep}{1.5pt}
\renewcommand{\arraystretch}{1.15}
\renewcommand{\tabularxcolumn}[1]{m{#1}}
\begin{tabularx}{\linewidth}{@{}l*{2}{>{\centering\arraybackslash}X}c@{}}
\toprule
\textbf{Method} & \textbf{HumanEval+} & \textbf{LiveCodeBench} & \textbf{Avg.} \\
\midrule
Student & 64.0 & 21.1 & 42.6 \\
Teacher & 78.0 & 46.5 & 62.3 \\
\midrule
Vanilla OPD & 75.4 & 43.7 & 59.6 \\
OPDVR & \underline{76.2} & \underline{44.2} & \underline{60.2} \\
\midrule
\rowcolor{ourblue}
\textbf{R$^{2}$-OPD} & \textbf{77.3} & \textbf{45.1} & \textbf{61.2} \\
\bottomrule
\end{tabularx}
\end{wraptable}

% Start prose directly: do not insert \noindent here.
We further evaluate R$^{2}$-OPD on code generation by distilling an separate RL-trained Qwen3-4B-Non-Thinking teacher into a Qwen3-1.7B-Base student. Both teacher RL and student OPD use TACO~\citep{li2023taco}, and we report $\mathrm{avg}@4$ on HumanEval+~\citep{liu2023evalplus} and LiveCodeBench~\citep{jain2025livecodebench}. As shown in Table~\ref{tab:coding_results_17b}, R$^{2}$-OPD outperforms vanilla OPD and OPDVR~\citep{lin2026opdvr} on both benchmarks, achieving an average $\mathrm{avg}@4$ of $61.2\%$ with gains of $+1.6$ and $+1.0$, respectively. These results extend the observed benefits beyond mathematical reasoning.
\RtwoEndWrap
\endgroup

\Needspace{16\baselineskip}
\subsection{Ablation Studies}
\label{subsec:exp_ablations}
\label{subsec:exp_mechanism}

\begingroup
\setlength{\intextsep}{2pt}
\setlength{\columnsep}{1em}
\setlength{\parindent}{0pt}
% True wrapping table. Do not place this file inside a minipage or table.
\begin{wraptable}{r}{0.50\columnwidth}
\vspace{-6pt} % 新增：上移 caption 和表体
\centering
\setlength{\abovecaptionskip}{0pt}
\setlength{\belowcaptionskip}{4pt}
\caption{Accuracy (\%) for ablation studies and GRPO composition.
$\Delta_{\mathrm{avg}}$ denotes the change in mean accuracy across
 the three benchmarks relative to full R$^{2}$-OPD.}
\label{tab:ablation_results}
\label{tab:mechanism_validation}
\scriptsize
\setlength{\tabcolsep}{1pt}
\renewcommand{\arraystretch}{1.12}
\renewcommand{\tabularxcolumn}[1]{m{#1}}
\begin{tabularx}{\linewidth}{
 @{}>{\raggedright\arraybackslash}p{0.32\linewidth}
 *{3}{>{\centering\arraybackslash}X}c@{}
}
\toprule
\textbf{Variant} & \textbf{AIME24} & \textbf{MATH500}
 & \textbf{Minerva} & $\Delta_{\mathrm{avg}}$ \\
\midrule
\rowcolor{ourblue}
\textbf{R$^{2}$-OPD (full)} & 30.1 & 77.8 & 29.4 & -- \\
\midrule
Reversed alignment & 25.4 & 73.5 & 26.9 & $-3.8$ \\
Outcome-permuted & 27.9 & 75.9 & 28.5 & $-1.7$ \\
Sign-only gate & 27.3 & 76.6 & 28.7 & $-1.6$ \\
Magnitude-only gate & 26.1 & 75.8 & 28.2 & $-2.4$ \\
No mass-norm & 28.1 & 76.7 & 29.8 & $-0.9$ \\
\midrule
R$^{2}$-OPD + GRPO & 31.5 & 78.0 & 30.2 & $+0.8$ \\
\bottomrule
\end{tabularx}
\end{wraptable}

Table~\ref{tab:ablation_results} examines reward alignment, disagreement magnitude, and per-response coefficient-mass normalization, using one training run per variant. Reversed alignment negates $z(q,o)A_t$ in the gate without changing coefficient signs, yielding the largest average accuracy drop of $3.8$. The sharpness sweep retains strong performance across the tested values from $10^{-4}$ to $10^{-1}$ (Appendix~\ref{app:exp_sensitivity}). Together, these observations support reward-consistent allocation without requiring narrowly tuned gate sharpness. The outcome-permuted control shuffles labels across complete responses while preserving per-response mass and the marginal outcome distribution. Its $1.7$ drop supports the value of correctly matched outcome feedback beyond mass preservation. The sign-only gate removes magnitude dependence while retaining the original advantages, reducing average accuracy by $1.6$; the magnitude-only gate removes outcome alignment and incurs a larger decline of $2.4$. These controls favor the full outcome- and magnitude-dependent configuration under the shared settings. Removing mass normalization while retaining the response-local scale reduces accuracy by $0.9$ points, supporting normalization in this configuration. Adding GRPO improves all three benchmarks with an average gain of $+0.8$, suggesting complementarity with direct reward optimization. Control definitions and the $\beta$ sensitivity analysis are provided in Appendices~\ref{app:exp_ablations} and~\ref{app:exp_sensitivity}, respectively.
\RtwoEndWrap
\endgroup

\begin{figure}[!tbp]
    \centering
    \includegraphics[width=0.82\linewidth]
        {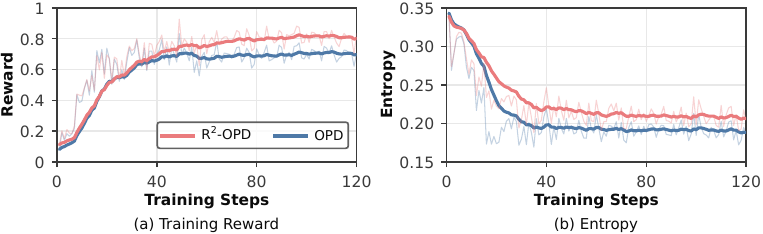}
    \caption{
        Training reward and student policy entropy
        for R$^{2}$-OPD and standard OPD.
        Thin lines show raw observations;
        thick lines show smoothed trends.
    }
    \label{fig:reward_entropy_training_dynamic}
\end{figure}

\subsection{Training Dynamics}
\label{subsec:exp_training_dynamics}

To examine how reward-aligned supervision allocation affects the distillation process, we track training reward and student policy entropy. Figure~\ref{fig:reward_entropy_training_dynamic} compares R$^{2}$-OPD with standard sampled-token OPD, denoted as OPD in the figure. Both methods improve training reward, but R$^{2}$-OPD maintains a higher smoothed reward trajectory over most of the plotted interval, with a separation that persists into late training. The raw observations fluctuate and overlap, so this comparison describes the overall trend rather than an advantage at every update.
Policy entropy decreases for both methods, while R$^{2}$-OPD retains higher entropy during the middle and later stages. Thus, its higher reward trajectory coexists with higher policy entropy, rather than a greater reduction in entropy.

\section{Related Work}

\paragraph{On-policy distillation of language models.}
Knowledge distillation transfers a teacher distribution into a smaller student \citep{hinton2015distilling}. For autoregressive language models, MiniLLM adopts the mode-seeking reverse KL \citep{gu2024minillm}, GKD reduces train--test mismatch by querying the teacher on student-generated prefixes \citep{agarwal2024gkd}, and DistiLLM studies divergence choice and adaptive reuse of student rollouts \citep{ko2024distillm}. A common scalable formulation of on-policy distillation casts the detached teacher--student log-probability gap as a per-token reward in a policy-gradient-style objective \citep{lu2025onpolicy}. Such objectives have been adopted in large-scale reasoning post-training \citep{yang2025qwen3,xiao2026mimo,deepseekai2026v4}; see \citet{song2026survey} for a broader survey. Recent work studies OPD optimization and failure modes \citep{li2026rethinking,fu2026revisiting}, improves sampled-token estimation or target shaping \citep{oh2026vopd,jang2026stable}, and rebalances training data \citep{hou2026uniopd}. G-OPD generalizes OPD as dense KL-constrained RL; its ExOPD variant amplifies teacher--reference log-ratio rewards relative to KL regularization~\citep{yang2026gopd}.

\paragraph{Selective token supervision in OPD.}
Complementary work asks whether OPD should supervise every visited position equally. TSD-KD applies direct distribution matching to selected tokens and weaker preference feedback elsewhere \citep{kim2026tsdkd}. Other methods select or reweight tokens using teacher--student disagreement, uncertainty, or support compatibility \citep{xu2026tip,wang2026teachability,xing2026tropd,jin2026entropyaware}; asymmetric or hierarchical schemes distinguish positive- and negative-advantage regions or combine trajectory filtering with token weighting \citep{jia2026aopd,li2026fireopd}. EGRSD incorporates outcome information by coupling reward-grounded update directions with teacher-derived magnitudes and confidence gating \citep{ke2026respecting}.

\paragraph{Outcome-guided distillation and RLVR.}
RLVR complements dense OPD feedback with outcome-level rewards~\citep{shao2024deepseekmath,guo2025deepseekr1}. Existing methods combine or sequence these objectives~\citep{xu2025kdrl,xiao2026mimo,li2026sequential}, route samples between them~\citep{li2026srpo}, or incorporate teacher information and verifier feedback into RL and self-distillation~\citep{wang2026distilledrl,yang2026rlsdvr,cai2026h2sd}. Closely related, SG-OPD combines token-level sign-consistency routing with phased sampling of verified teacher responses~\citep{xu2026sgopd}; and OPDVR masks outcome-inconsistent corrections~\citep{lin2026opdvr}. Related diagnostics examine outcome-dependent local supervision~\citep{ma2026outcome}. R$^{2}$-OPD continuously reweights the original distillation coefficients using outcome alignment and disagreement magnitude.

\section{Conclusion}
\label{sec:conclusion}

We revisit uniform supervision in on-policy distillation through the mismatch between \emph{local teacher preference} and \emph{student continuation value}. Our analysis shows that local imitation can reduce reverse KL while lowering expected reward under a fixed student continuation policy, providing a unified account of over-correction and over-trust.
This motivates R$^{2}$-OPD, which uses trajectory outcomes as evidence for supervision allocation and continuously reweights teacher corrections using reward alignment and disagreement magnitude.
The method preserves dense feedback, correction signs, and per-response absolute coefficient mass. We establish sufficient reward--utility agreement conditions under which reallocation improves first-order task progress over uniform OPD.
Experiments across student sizes and task domains demonstrate consistent gains over standard OPD in mathematical reasoning and code generation, while ablations support the core design.
Together, these findings highlight that effective OPD depends not only on \emph{where} teacher feedback is collected, but also on \emph{how its influence is allocated} using evidence of student task success.

% Unnumbered statements after the main text and before references.

\section*{Reproducibility Statement}

Section~\ref{sec:method} and Algorithm~\ref{alg:r2opd} specify the reweighting rule and the training procedure; Appendix~\ref{app:exp_reproducibility} reports the training configurations, computing environment, and evaluation protocols, and Appendix~\ref{app:exp_ablations} defines the ablations. Assumptions and derivations appear in Appendices~\ref{app:proofs} and~\ref{app:motivation-proofs}, and the task-progress analysis in Appendix~\ref{app:r2_task_progress}. 
% A reference implementation---the R$^{2}$-OPD coefficient, its ablations, the outcome-permuted control, the compared baselines, and the tests that pin them---is available at \url{https://anonymous.4open.science/r/r2-opd-62F4/}; it is extracted verbatim from the code that produced our runs, checked against it numerically, and depends only on PyTorch.
\section*{AI Use Statement}

Large language models assisted with literature search, manuscript
drafting and editing, checks of mathematical derivations, and \LaTeX{}
formatting. The authors designed and conducted the experiments,
reviewed the AI-assisted content, and take full responsibility for the
methodology, theoretical claims, empirical results, and final manuscript.
\clearpage

% References.
\bibliography{iclr2027_conference}
\bibliographystyle{iclr2027_conference}

% Appendices.
\clearpage
\appendix
% Replaces the attached appendix_a_impletation.tex.
% Input once after \appendix; do not add another surrounding \section.
\section{Implementation Details}
\label{app:exp_reproducibility}

\begin{table}[H]
\centering
\begin{minipage}{0.50\linewidth}
\centering
\caption{Training settings and mathematical evaluation.}
\label{tab:exp_hyperparameters}

\small
\setlength{\tabcolsep}{3pt}
\renewcommand{\arraystretch}{1.10}
\begin{tabularx}{\linewidth}{
    @{}>{\raggedright\arraybackslash}Xr@{}
}
\toprule
\textbf{Parameter} & \textbf{Value} \\
\midrule
\multicolumn{2}{c}{\emph{Optimization}} \\
\midrule
Optimizer & AdamW \\
Learning rate ($\eta$) & $1\times10^{-6}$ \\
Schedule / warmup & Constant / none \\
AdamW $(\beta_1,\beta_2)$ & $(0.9,\,0.99)$ \\
AdamW $\epsilon$ & $10^{-8}$ \\
Weight decay & $0.1$ \\
Global $\ell_2$-norm clip & $1.0$ \\
Accumulation and all-reduce precision & FP32 \\

\midrule
\multicolumn{2}{c}{\emph{Distillation loss}} \\
\midrule
Clipping $(\epsilon_{\mathrm{low}},\epsilon_{\mathrm{high}})$
& $(0.2,\,0.2)$ \\
OPD multiplier ($c_{\mathrm{OPD}}$) & $1$ \\
Scale floor ($\epsilon$) & $2^{-23}$ \\
Dual clipping & Disabled \\
Entropy / reference-KL coefficients & $0\,/\,0$ \\

\midrule
\multicolumn{2}{c}{\emph{Training and rollout}} \\
\midrule
Prompts per rollout batch & 64 \\
Responses per prompt & 4 \\
Optimizer batch (responses) & 256 \\
Micro-batch tokens per device & $\leq10{,}240$ \\
Rollout temperature & $1.0$ \\
Rollout top-$p$ / top-$k$ & $1$ / disabled \\
Maximum prompt length & 2,048 tokens \\
Maximum response length & 8,192 tokens \\

\midrule
\multicolumn{2}{c}{\emph{Mathematical evaluation}} \\
\midrule
Samples per AIME24/25 problem & 16 \\
Samples per other problem & 4 \\
Temperature & $1.0$ \\
Top-$p$ & $0.95$ \\
\bottomrule
\end{tabularx}
\end{minipage}
\end{table}

\subsection{Training Configuration}
\label{app:exp_configuration}

Mathematical reasoning and code generation share the optimization
and rollout settings in Table~\ref{tab:exp_hyperparameters}.
Unless otherwise specified, controlled comparisons use the same
prompt budget.

\paragraph{Hardware and software.}
All experiments use a single machine with eight NVIDIA H800-80GB
GPUs interconnected via 400\,GB/s NVLink.
The software environment comprises Ubuntu~22.04.5, Python~3.11.15,
PyTorch~2.11.0, CUDA~12.9, NCCL~2.26.2, Ray~2.55.1,
SGLang~0.5.12.post1, Transformers~5.6.0, and
FlashAttention~2.7.4.post1.

\subsection{Policy Loss and Optimization Protocol}
\label{app:optimization_protocol}

\paragraph{Policy loss and reduction.}
For a batch $\mathcal B$ of $N$ responses, let $m_{i,t}$ mark valid
response tokens, $n_i=\sum_t m_{i,t}$,
$\widetilde n_i=\max\{n_i,1\}$, and
$\widehat A_{i,t}=\operatorname{sg}[c_{\mathrm{OPD}}A_{i,t}^{\mathrm{R2}}]$.
With a detached snapshot denominator, define
$\rho_{i,t}(\theta)=\pi_\theta(o_{i,t}\mid s_{i,t})/
\pi_{\bar\theta}(o_{i,t}\mid s_{i,t})$.
The implemented loss is
\begin{equation}
\begin{aligned}
 \widehat{\mathcal L}_{\mathrm{impl}}(\theta;\mathcal B)
 &=\frac{1}{N}\sum_{i=1}^{N}\frac{1}{\widetilde n_i}
   \sum_t m_{i,t}\,
   \ell_{\mathrm{clip}}(\rho_{i,t}(\theta),\widehat A_{i,t}),\\
 \ell_{\mathrm{clip}}(\rho,a)
 &=-\min\!\left\{\rho a,
 \operatorname{clip}(\rho,1-\epsilon_{\mathrm{low}},
 1+\epsilon_{\mathrm{high}})a\right\},
\end{aligned}
\label{eq:impl_actual_policy_loss}
\end{equation}
with numerical settings in Table~\ref{tab:exp_hyperparameters}.
This averages per-response token means rather than pooling tokens
across responses. Masks include generated tokens and terminal EOS
when present, exclude prompts and padding, and retain truncated responses;
empty masks contribute zero.
Distributed accumulation preserves the weights $1/(N\widetilde n_i)$.
No additional advantage whitening, centering, coefficient clipping, or
weight cap is applied. Truncated responses retain standard OPD
coefficients, and zero-mass responses retain zero coefficients.
Pure distillation uses outcomes only through the gate, without an
additional reward-derived advantage.

\paragraph{Rollout and update schedule.}
Each fresh rollout batch forms one optimizer batch, with gradients
accumulated across token-packed micro-batches before a single optimizer
step, without replay or additional optimization epochs.
Student weights are synchronized before sampling and remain unchanged
until the update.
The trainer recomputes detached snapshot log probabilities for both
$A_{i,t}$ and the ratio denominator; teacher scores and reweighting
coefficients remain fixed throughout accumulation.
Engine--trainer log-probability discrepancies are logged without
importance-sampling correction.
With consistent trainer evaluations, $\rho_{i,t}(\bar\theta)=1$,
so clipping is inactive and the policy-loss gradient matches the detached
log-probability surrogate with the same response-wise reduction.
Its relation to the token-summed analysis is given in
Appendix~\ref{app:r2_policy_loss}.

\subsection{Evaluation and Reporting} \label{app:exp_evaluation} 
Using the benchmark-specific sample count $k_{\mathcal E}$ and sampling settings in Table~\ref{tab:exp_hyperparameters}, we report \begin{equation} \mathrm{Acc}^{(s)}(\mathcal E) =\frac{100}{|\mathcal E|k_{\mathcal E}} \sum_{q\in\mathcal E}\sum_{j=1}^{k_{\mathcal E}} R(q,o_{q,j}^{(s)}). \label{eq:exp_avgk} \end{equation} For the main mathematical comparison, we report mean $\pm$ sample standard deviation across three training seeds for 1.7B cross-size distillation and single-run results for 4B same-size distillation. Reference checkpoints are evaluated once. The seven-benchmark macro-average is computed within each run before cross-seed aggregation.

\section{Ablation Experiments}
\label{app:exp_ablations}

Unless otherwise specified, the controls in
Table~\ref{tab:ablation_results} use the same training and
evaluation settings as full R$^{2}$-OPD.

\subsection{Ablation Definition}
\label{app:def_ablation}

\paragraph{Response-level normalization.}
For a complete response with $M_o=\sum_t|A_t|>0$, each modified
gate $g_t^{(v)}$ uses a recomputed normalizer:
\begin{equation}
 Z_o^{(v)}=\frac{M_o}{\sum_j g_j^{(v)}|A_j|},
 \qquad A_t^{(v)}=Z_o^{(v)}g_t^{(v)}A_t.
\label{eq:app_ablation_variant_coeff}
\end{equation}
This preserves coefficient mass under each gate modification;
no mass-norm is the exception.
All advantages and reweighting factors are detached.
Zero-mass responses retain zero coefficients with $Z_o=1$;
truncated responses retain standard OPD coefficients.
Below, $z=z(q,o)$ and $\nu_o$ follows Eq.~\eqref{eq:r2_scale}.

\paragraph{Reversed alignment.}
Negating the alignment signal gives
\begin{equation}
 g_t^{\mathrm{rev}}
 =\sigma\!\left(-\beta\frac{zA_t}{\nu_o}\right).
\label{eq:app_ablation_inverse}
\end{equation}
This reverses the gate's relative emphasis on aligned and
inconsistent corrections without changing $A_t$ or verifier labels.

\paragraph{Outcome-permuted.}
Within each rollout batch, we permute outcome labels across complete
responses while retaining their advantages, scales, and token masks.
For a permutation $\varpi$ of these responses, the gate is
\begin{equation}
 g_{i,t}^{\mathrm{perm}}
 =\sigma\!\left(
 \beta\frac{z_{\varpi(i)}A_{i,t}}{\nu_{o_i}}
 \right).
\label{eq:app_ablation_outcome_permuted}
\end{equation}
The normalizer is recomputed using
Eq.~\eqref{eq:app_ablation_variant_coeff}, preserving per-response
coefficient mass and the marginal outcome distribution.
Truncated responses retain standard OPD coefficients.

\paragraph{Sign-only gate.}
We remove magnitude dependence from the gate by assigning one fixed level per
alignment group:
\begin{equation}
g^{\mathrm{sign}}_t =
\begin{cases}
3/4, & z A_t > 0,\\
1/4, & z A_t < 0,\\
1/2, & z A_t = 0,
\end{cases}
\end{equation}
so the gate reads neither $\beta$ nor $\nu_o$, while the original $A_t$ remains in
Eq.~\eqref{eq:ablation_norm}. Let $a$ denote the aligned share of the response's
coefficient mass. Normalization gives multipliers $Z^{\mathrm{sign}}_o
g^{\mathrm{sign}}_t = 3/(1+2a)$ on aligned positions and $1/(1+2a)$ on inconsistent
ones, so aligned corrections receive exactly three times the multiplier of
inconsistent ones, within $[1,3]$ and $[1/3,1]$ respectively. This control removes
magnitude dependence at a fixed reallocation ratio; it does not match the full
gate's reweighting strength, which for two positions of equal magnitude and
opposite alignment is exactly $\exp(\beta|A_t|/\nu_o)$, about $1.001$ at a
median-magnitude position under the shared $\beta = 10^{-3}$.

\paragraph{Magnitude-only gate.}
We retain disagreement magnitude without outcome alignment:
\begin{equation}
 g_t^{\mathrm{mag}}
 =\sigma\!\left(\beta\frac{|A_t|}{\nu_o}\right).
\label{eq:app_ablation_magnitude}
\end{equation}
Weighting depends only on normalized disagreement magnitude;
the final coefficients retain their original signs.

\paragraph{No mass-norm.}
Setting $Z_o=1$ while retaining the full gate and scale gives
\begin{equation}
 A_t^{\mathrm{no\mbox{-}norm}}=g_tA_t.
\label{eq:app_ablation_no_norm}
\end{equation}
This removes coefficient-mass restoration, not the scale $\nu_o$.
It changes coefficient scale and is therefore not a direction-only
comparison.

\paragraph{Composition with GRPO.}
We replace the OPD component in OPD + GRPO with the full
R$^{2}$-OPD coefficient:
\begin{equation}
 A_t^{\mathrm{joint}}
 =A_t^{\mathrm{GRPO}}+\lambda_{\mathrm{OPD}}A_t^{\mathrm{R2}},
 \qquad \lambda_{\mathrm{OPD}}=1.
\label{eq:app_ablation_grpo}
\end{equation}
The GRPO advantage is computed unchanged and excluded from the
R$^{2}$ gate and response-level coefficient-mass normalization.
The combined detached advantage is inserted into the clipped
policy surrogate using the mean of per-response token means,
as specified in Appendix~\ref{app:optimization_protocol}.
All OPD coefficients are detached. This control tests the
complementarity between direct reward optimization and reward-aligned
distillation.

\subsection{Outcome-Branch Controls}
\label{app:exp_controls}

Table~\ref{tab:outcome_branch_controls} compares full R$^{2}$-OPD
with reweighting restricted to successful or failed responses,
using one training run per variant.
For complete responses, the controls are defined as follows.

\paragraph{Success-only.}
We reweight successful responses and retain standard OPD on failures:
\begin{equation}
 A_t^{\mathrm{success}}=
 \begin{cases}
 A_t^{\mathrm{R2}}, & z=+1,\\
 A_t, & z=-1.
 \end{cases}
\label{eq:app_success_only}
\end{equation}
Relative to full R$^{2}$-OPD, this removes reweighting on failed
responses without discarding their teacher supervision.

\paragraph{Failure-only.}
We reweight failed responses and retain standard OPD on successes:
\begin{equation}
 A_t^{\mathrm{failure}}=
 \begin{cases}
 A_t, & z=+1,\\
 A_t^{\mathrm{R2}}, & z=-1.
 \end{cases}
\label{eq:app_failure_only}
\end{equation}
Relative to full R$^{2}$-OPD, this removes reweighting on successful
responses without discarding their teacher supervision.

Both controls use the full gate and per-response normalization
on the selected branch, without rescaling for branch frequency.
Truncated responses retain standard OPD coefficients.
The full reference uses the same experimental setting as the branch
controls.

\begin{table}[!t]
\centering
\caption{Accuracy (\%) on seven mathematical reasoning benchmarks for
outcome-branch controls. $\Delta_{\mathrm{avg}}$ reports the change in
the seven-benchmark average relative to full R$^{2}$-OPD.}
\label{tab:outcome_branch_controls}

\small
\setlength{\tabcolsep}{2.5pt}
\renewcommand{\arraystretch}{1.08}

\begin{tabularx}{\linewidth}{
    @{}l
    *{7}{>{\centering\arraybackslash}X}
    >{\centering\arraybackslash}X@{}
}
\toprule
Variant
& AIME24
& AIME25
& AMC23
& AMC24
& MATH500
& Minerva
& Olympiad
& $\Delta_{\mathrm{avg}}$ \\
\midrule

\rowcolor{ourblue}
R$^{2}$-OPD (full)
& 30.1 & 24.5 & 65.0 & 54.4
& 77.8 & 29.4 & 45.9
& -- \\

Success-only
& 31.3 & 22.3 & 63.7 & 50.0
& 76.3 & 29.1 & 45.8
& -1.2 \\

Failure-only
& 28.1 & 20.0 & 63.7 & 52.2
& 75.1 & 28.7 & 43.7
& -2.2 \\

\bottomrule
\end{tabularx}
\end{table}

\paragraph{Results and interpretation.}
Full R$^{2}$-OPD achieves the highest seven-benchmark average of
$46.7$, exceeding Success-only and Failure-only by $1.2$ and
$2.2$, respectively.
The larger decline of Failure-only suggests that removing
reweighting on successful responses is more detrimental in this
setting.
Together, these results support reweighting both outcome branches,
consistent with their complementary roles in addressing
over-correction and over-trust.
The comparisons reflect branch-level effects under the observed
outcome distribution, not causal token-level utility.
\section{Hyperparameter Sensitivity}
\label{app:exp_sensitivity}

We examine gate sharpness in the Qwen3-4B-GRPO $\rightarrow$
Qwen3-1.7B-Base setting through an accuracy sweep, a fixed-rollout
coefficient analysis, and training-time mass profiles.
All settings retain the full reweighting rule, with the response-level
normalizer recomputed for each $\beta$.
Accuracy results use one training run per setting; the default
$\beta=10^{-3}$ is shared across the main experiments without
experiment-specific tuning.

\begin{figure}[!htbp]
    \centering
    \includegraphics[width=0.90\linewidth]
        {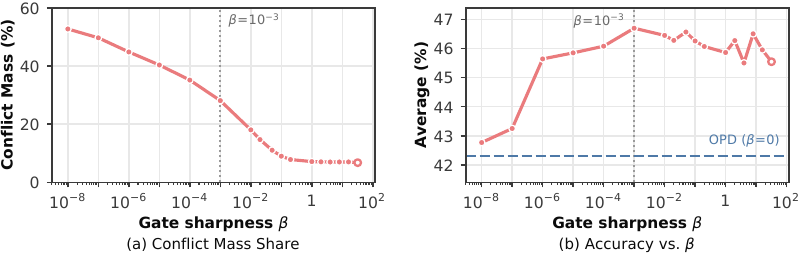}
    \caption{
        Gate-sharpness analysis.
        (a) Mean response-level inconsistent coefficient-mass share on a
        fixed rollout batch, holding advantages, outcomes, response-local
        scales, and token masks constant across $\beta$.
        Responses are equally weighted, including identity-fallback responses.
        (b) Average mathematical accuracy from one training run per setting.
        Vertical dotted lines mark $\beta=10^{-3}$; the horizontal dashed
        line in (b) denotes standard OPD, corresponding to $\beta=0$.
    }
    \label{fig:beta_sensitivity_analysis}
\end{figure}

\subsection{Sensitivity of Task Performance}
\label{app:beta_performance}

Figure~\ref{fig:beta_sensitivity_analysis}(b) shows that accuracy varies
non-monotonically with gate sharpness.
It increases from approximately $42.8\%$ at $\beta=10^{-8}$ to
$46.7\%$ at $\beta=10^{-3}$, the highest observed score in the sweep.
The sampled values between $10^{-4}$ and $10^{-1}$ achieve roughly
$46.1$--$46.7\%$, whereas those with $\beta\geq1$ yield approximately
$45.5$--$46.5\%$.
All tested settings exceed the displayed OPD reference.
Thus, strong performance extends beyond a single sharpness value,
but stronger gating offers no consistent improvement over the default.

\subsection{Reward-Inconsistent Mass Allocation}
\label{app:beta_inconsistent_mass}

For a response with $M_o>0$, $\kappa_o(\beta)$ is the fraction of its
absolute coefficient mass assigned to positions with $z(q,o)A_t<0$
(Eq.~\eqref{eq:r2_diagnostics}).
For a rollout batch $\mathcal B$ of $N=|\mathcal B|$ responses, we report
$\bar\kappa_{\mathcal B}(\beta)=N^{-1}\sum_{o\in\mathcal B}\kappa_o(\beta)$,
using equal response weights $\omega_o=1/N$ rather than pooling
coefficient mass across responses.
Truncated responses routed to the vanilla fallback remain in the
average. Their original coefficients give
$\kappa_o(\beta)=\kappa_o(0)$ for every $\beta$; their outcomes are used
for this diagnostic but not for reweighting.
% Author check: The ratio above is defined for M_o>0. The implementation's
% handling of M_o=0 or empty responses was not supplied. Do not state that
% such responses are excluded, assigned zero, or absent without verification.
% If present, document the actual logged-value convention and ensure that
% both statistics retain the stated full-batch denominator.

\paragraph{Fixed-rollout reallocation.}
Figure~\ref{fig:beta_sensitivity_analysis}(a) applies the reweighting
rule and its fallback to identical rollouts.
The mean inconsistent share decreases from approximately $53\%$ at
$\beta=10^{-8}$ to $28\%$ at the default, approaching $7\%$ at large
sharpness.
The default thus lowers the mean response-level share by about $25$
percentage points with batch membership and original coefficients fixed.
Unchanged fallback contributions remain part of the residual share.

\paragraph{Paired training-time analysis.}
Figure~\ref{fig:conflict_mass_dynamics} compares original and applied
coefficients on each R$^{2}$-OPD rollout batch.
R$^{2}$-OPD (natural) uses the original $A_t$, corresponding to
$\bar\kappa_{\mathcal B}(0)$ on that batch, not a separate training run.
Its trajectory closely tracks independently trained OPD, indicating
similar mean pre-reweighting conflict shares.
From approximately step~50 onward, the natural share remains near
$50$--$55\%$, whereas the applied share is $25$--$30\%$.
The paired curves use identical responses, advantages, outcomes, masks,
and equal response weights; fallback responses contribute identically
to both curves.
Their gap therefore isolates the direct reallocation effect on the
mean within-response share, which is approximately halved while
per-response coefficient mass is preserved.

\begin{figure}[!htbp]
    \centering
    \includegraphics[width=0.56\linewidth]
        {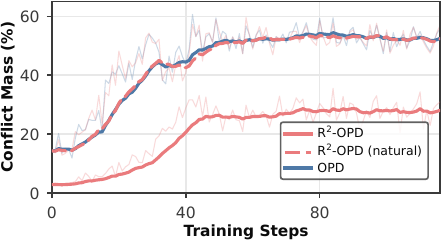}
    \caption{
        Mean response-level inconsistent coefficient-mass share during
        4B-to-1.7B distillation, averaged equally over all batch responses,
        including identity-fallback responses.
        R$^{2}$-OPD (red solid) uses the applied coefficients at
        $\beta=10^{-3}$; R$^{2}$-OPD (natural, red dashed) uses the original
        coefficients on the same rollouts.
        OPD (blue solid) denotes an independent standard-OPD run.
        Thin traces show raw observations; thick curves show smoothed trends.
    }
    \label{fig:conflict_mass_dynamics}
\end{figure}

\paragraph{Effective weighting and coefficient redistribution.}
For $M_o>0$, let $p_o(t)=|A_t|/M_o$.
On reweighted responses, the effective multiplier at nonzero-advantage
positions is
\begin{equation}
 w_t(\beta)
 =\frac{A_t^{\mathrm{R2}}(\beta)}{A_t}
 =\frac{g_t(\beta)}{\mathbb E_{p_o}[g(\beta)]}.
\label{eq:beta_effective_multiplier}
\end{equation}
On identity-fallback responses, $w_t(\beta)=1$.
Thus, effective reweighting depends on normalized disagreement and its
coefficient-mass distribution, not on $\beta$ alone.
Near-uniform gating at a median-magnitude token does not imply
near-identity weighting over the response.

Using the same full-batch denominator, define
\begin{equation}
 D_{\mathcal B}(\beta)
 =\frac{1}{N}\sum_{o\in\mathcal B}d_o(\beta),
 \qquad
 d_o(\beta)=\mathbb E_{p_o}\!\left[|w_t(\beta)-1|\right].
\label{eq:beta_coefficient_displacement}
\end{equation}
Fallback responses have $d_o=0$ but remain counted in $N$.
If their batch fraction is $f_{\mathcal B}<1$, then
$D_{\mathcal B}=(1-f_{\mathcal B})\bar d_{\mathrm{nonfb}}$,
where $\bar d_{\mathrm{nonfb}}$ is the mean displacement over the
remaining responses.
Mass conservation holds on both reweighted and fallback responses, so
the triangle inequality gives
\begin{equation}
 D_{\mathcal B}(\beta)
 \geq 2\left|\bar\kappa_{\mathcal B}(\beta)
                  -\bar\kappa_{\mathcal B}(0)\right|.
\label{eq:beta_mass_displacement_bound}
\end{equation}
Proposition~\ref{prop:r2_path} gives nonincreasing shares on reweighted
responses, while fallback shares are constant.
Hence $\bar\kappa_{\mathcal B}(0)\geq
\bar\kappa_{\mathcal B}(10^{-8})$, and the reported shares imply
$D_{\mathcal B}(10^{-3})\gtrsim0.50$.
This is a lower bound on the mean response-normalized coefficient
displacement across the full batch, including unchanged fallback
responses, not on a pooled mass ratio or parameter-gradient displacement.
It does not identify the mass-weighted tail of normalized disagreement
or attribute accuracy gains to particular token ranges.

\paragraph{Implications for supervision allocation.}
The fixed-rollout analysis shows substantial within-response reallocation
at the default, while the accuracy sweep shows that stronger
suppression of inconsistent mass does not consistently improve
performance.
This distinction is consistent with Eq.~\eqref{eq:r2_task_gain}:
first-order task progress depends on the association between gate
weights and correction utility, not on inconsistent mass alone.
The results support calibrating correction strength using outcome
evidence, without treating maximal outcome selectivity as the goal.
\section{Additional Training Dynamics}
\label{app:additional_training_dynamics}

\subsection{Response Length}
\label{app:response_length}

Figure~\ref{fig:response_length} compares training-time response
lengths for R$^{2}$-OPD and standard sampled-token OPD.
Both methods exhibit rapid length growth during early training,
reaching approximately $5.5$k tokens by step~40.
Their trajectories remain closely matched thereafter, including
similar fluctuations during later training.
The performance gains therefore coexist with comparable response
lengths, providing evidence against longer training rollouts as
their primary explanation.

\begin{figure}[!htbp]
    \centering
    \includegraphics[width=0.50\linewidth]
        {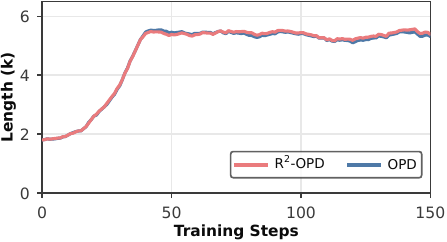}
    \caption{
        Training-time response length for R$^{2}$-OPD and standard
        sampled-token OPD, denoted as OPD.
        Length is measured in thousands of tokens.
    }
    \label{fig:response_length}
\end{figure}

\subsection{Evaluation Trajectories Across All Benchmarks}
\label{app:full_evaluation_trajectories}

Figure~\ref{fig:eval_acc_all_benchmarks_17b} extends the learning curves
in Figure~\ref{fig:main_results_accuracy_17b} to all seven mathematical
benchmarks in the 1.7B cross-size setting.
Most methods improve rapidly within the first $40$--$60$ training steps
and change more gradually thereafter.
R$^{2}$-OPD shows early and sustained gains on several benchmarks,
with particularly visible separation on AIME24, AIME25, AMC24,
and OlympiadBench.
These benchmark-wise trajectories complement the reported results
in Table~\ref{tab:main_results} and reveal variation in the magnitude
of the gains.

\begin{figure}[!p]
    \centering
    \includegraphics[
        width=\linewidth,
        height=0.82\textheight,
        keepaspectratio
    ]{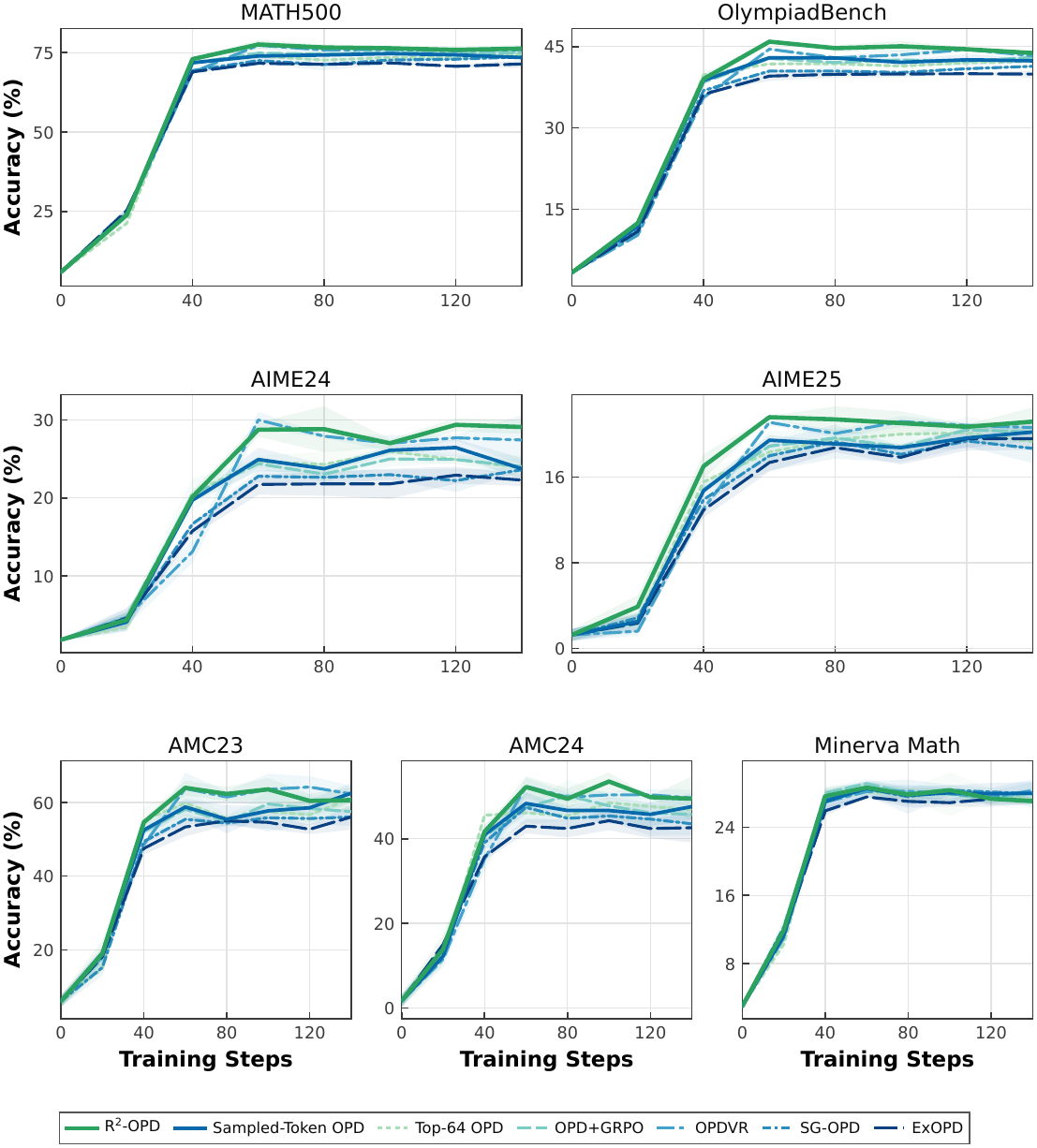}
    \caption{
        Evaluation accuracy across seven mathematical reasoning benchmarks
        during cross-size distillation from Qwen3-4B-GRPO to
        Qwen3-1.7B-Base.
        Curves show the mean and shaded bands indicate $\pm$ one sample
        standard deviation across three independent training seeds.
    }
    \label{fig:eval_acc_all_benchmarks_17b}
\end{figure}
\section{Computational Cost}
\label{app:exp_efficiency}

\paragraph{Model computation.}
For matched rollout batches and update schedules, R$^{2}$-OPD uses
the same student generation, teacher scoring, and student optimization
as standard sampled-token OPD.
Reweighting introduces no additional model passes, trainable parameters,
or inference-time components.
Neither teacher-generated responses nor auxiliary value estimation
is required.

\paragraph{Reweighting overhead.}
For a batch containing $N$ valid response tokens, let $\mathcal I$
index complete positive-mass responses and $n_i$ denote their lengths.
The additional coefficient computation has cost
\begin{equation}
 C_{\mathrm{rw}}
 =\sum_{i\in\mathcal I}C_{\mathrm{med}}(n_i)+O(N),
\label{eq:exp_reweighting_cost}
\end{equation}
where $C_{\mathrm{med}}$ denotes the implementation-dependent median
cost; gate evaluation and mass normalization are linear in token count.
Reweighting can be implemented with $O(N)$ auxiliary scalar storage,
excluding median workspace, without additional vocabulary-sized tensors.
All reweighting operations are detached from backpropagation.

\paragraph{Verification and end-to-end cost.}
Outcome verification adds task-dependent cost when the baseline does
not already score the same rollouts; otherwise, verifier outputs can
be reused.
Online teacher scoring belongs to distillation cost, whereas teacher
preparation and benchmark evaluation are separate costs.
End-to-end runtime also depends on realized response lengths and
implementation, so unchanged model-pass counts alone do not establish
wall-clock equivalence or negligible overhead.
% Replace the previous appendix_a_formal_derivations.tex; input only once.
% Main-text labels are resolved by the complete manuscript.
% The batch proof uses the log-probability surrogate in the current main text.

\section{Gradient Analysis of Sampled-Token OPD}
\label{app:proofs}

This appendix derives Eq.~\eqref{eq:opd_gradient_consistency}
and clarifies the coefficient-level meaning of reward alignment.
We adopt the assumptions of Section~\ref{sec:preliminaries} and
assume the regularity required to interchange differentiation and
expectation. Rollout distributions and sampled coefficients are fixed
during differentiation. The analysis concerns the unclipped surrogate,
not subsequent optimizer transformations.

\subsection{Statewise Reverse-KL Gradient}
\label{app:statewise_rkl}

At a fixed prefix $s$, write
\begin{equation}
 K_s(\theta)=\sum_{a\in\mathcal V}\pi_\theta(a\mid s)
 \log\frac{\pi_\theta(a\mid s)}{\pi_T(a\mid s)}.
\label{eq:app_rkl}
\end{equation}
Differentiation gives
\begin{align}
 \nabla_\theta K_s(\theta)
 &=\sum_a \nabla_\theta\pi_\theta(a\mid s)
   \log\frac{\pi_\theta(a\mid s)}{\pi_T(a\mid s)}
   +\sum_a\nabla_\theta\pi_\theta(a\mid s)
   \nonumber\\
 &=\mathbb E_{a\sim\pi_\theta(\cdot\mid s)}\!\left[
   \log\frac{\pi_\theta(a\mid s)}{\pi_T(a\mid s)}
   \nabla_\theta\log\pi_\theta(a\mid s)\right],
\label{eq:app_rkl_gradient}
\end{align}
where $\sum_a\nabla_\theta\pi_\theta(a\mid s)=0$ by normalization.
Define $A_\theta(s,a)=\log\pi_T(a\mid s)-\log\pi_\theta(a\mid s)$,
so that $A_t=A_{\bar\theta}(s_t,o_t)$ in
Eq.~\eqref{eq:sampled_advantage}.
Then
\begin{equation}
 \nabla_\theta K_s(\theta)
 =-\mathbb E_{a\sim\pi_\theta(\cdot\mid s)}\!\left[
 A_\theta(s,a)\nabla_\theta\log\pi_\theta(a\mid s)\right].
\label{eq:app_rkl_rho}
\end{equation}
Equation~\eqref{eq:app_rkl_rho} gives an unbiased single-action
gradient estimator. Separately,
\begin{equation}
 \mathbb E_{a\sim\pi_\theta(\cdot\mid s)}[A_\theta(s,a)]
 =-K_s(\theta)\leq0.
\label{eq:app_offset}
\end{equation}
A nonpositive mean coefficient does not imply uniform suppression:
the update depends on its association with the action-dependent
score function, whose expectation is zero.

\subsection{First-Order Equivalence of the Sampled Surrogate}
\label{app:sampled_surrogate_proof}

Let $\mathbb E_{\bar\theta}$ denote expectation over
$q\sim\mathcal D$ and $o\sim\pi_{\bar\theta}(\cdot\mid q)$.
The main-text surrogate is
\begin{equation}
 \mathcal L_{\mathrm{OPD}}^{\mathrm{sample}}(\theta;\bar\theta)
 =-\mathbb E_{\bar\theta}\!\left[
 \sum_{t=1}^{|o|}\operatorname{sg}[A_t]
 \log\pi_\theta(o_t\mid s_t)\right].
\label{eq:app_sampled_surrogate}
\end{equation}

\begin{proposition}[Rollout-point gradient equivalence]
\label{prop:app_unbiased}
With the rollout distribution fixed,
\begin{equation}
 \left.\nabla_\theta
 \mathcal L_{\mathrm{OPD}}^{\mathrm{sample}}(\theta;\bar\theta)
 \right|_{\theta=\bar\theta}
 =\left.\nabla_\theta
 \mathcal L_{\mathrm{OPD}}(\theta;\bar\theta)
 \right|_{\theta=\bar\theta}.
\label{eq:app_unbiased}
\end{equation}
\end{proposition}

\begin{proof}
Let $\psi_t=\left.\nabla_\theta\log\pi_\theta(o_t\mid s_t)
\right|_{\theta=\bar\theta}$.
Since $A_t$ is detached, the left-hand side equals
$-\mathbb E_{\bar\theta}[\sum_t A_t\psi_t]$.
Conditioning on each visited prefix and applying
Eq.~\eqref{eq:app_rkl_rho} gives
\begin{equation}
 -\mathbb E_{o_t\sim\pi_{\bar\theta}(\cdot\mid s_t)}
 [A_t\psi_t\mid s_t]
 =\left.\nabla_\theta K_t(\theta)\right|_{\theta=\bar\theta}.
\label{eq:app_conditional_gradient}
\end{equation}
Summing over response positions and taking the rollout expectation
therefore yields
\begin{align}
 -\mathbb E_{\bar\theta}\!\left[\sum_t A_t\psi_t\right]
 &=\mathbb E_{\bar\theta}\!\left[
 \sum_t\left.\nabla_\theta K_t(\theta)\right|_{\theta=\bar\theta}
 \right]\nonumber\\
 &=\left.\nabla_\theta
 \mathcal L_{\mathrm{OPD}}(\theta;\bar\theta)
 \right|_{\theta=\bar\theta},
\end{align}
which proves Eq.~\eqref{eq:opd_gradient_consistency}.
\end{proof}

The identity establishes gradient equivalence at the rollout parameters,
not equality of loss values or differentiation through state visitation.
It also requires the sampling distribution to agree with the
$\pi_{\bar\theta}$ used in $A_t$.
For comparison, the importance-ratio surrogate uses
$r_t(\theta)=\pi_\theta(o_t\mid s_t)/\pi_{\bar\theta}(o_t\mid s_t)$
in place of $\log\pi_\theta(o_t\mid s_t)$.
It has the same rollout-point gradient because
$\left.\nabla_\theta r_t(\theta)\right|_{\theta=\bar\theta}=\psi_t$,
but the two surrogates generally differ away from $\bar\theta$.

\subsection{Outcome-Aligned Coefficient Decomposition}
\label{app:alignment_partition}

For a fixed complete response, let $z=z(q,o)\in\{-1,+1\}$ and
let $m_t\in\{0,1\}$ mark valid response tokens.
Define the aligned and inconsistent components
\begin{equation}
 a_t^{\mathrm{ali}}=m_tA_t\mathbb I[zA_t>0],
 \qquad
 a_t^{\mathrm{inc}}=m_tA_t\mathbb I[zA_t<0].
\label{eq:app_partition_components}
\end{equation}
The two events partition all nonzero coefficients, giving
\begin{equation}
 m_tA_t=a_t^{\mathrm{ali}}+a_t^{\mathrm{inc}}.
\label{eq:app_exact_partition}
\end{equation}
For $A_t\neq0$,
$\mathbb I[zA_t>0]=(1+z\operatorname{sgn}(A_t))/2$.
Using $A_t\operatorname{sgn}(A_t)=|A_t|$ yields
\begin{equation}
 a_t^{\mathrm{ali}}=\frac{m_t}{2}(A_t+z|A_t|),
 \qquad
 a_t^{\mathrm{inc}}=\frac{m_t}{2}(A_t-z|A_t|).
\label{eq:app_closed_forms}
\end{equation}
Both identities also hold when $A_t=0$.

For $\ell_o(w;\theta)=-\sum_t\operatorname{sg}[w_t]
\log\pi_\theta(o_t\mid s_t)$, linearity in detached coefficients gives
\begin{align}
 \ell_o(mA;\theta)
 &=\ell_o(a^{\mathrm{ali}};\theta)
  +\ell_o(a^{\mathrm{inc}};\theta),\nonumber\\
 \nabla_\theta\ell_o(mA;\theta)
 &=\nabla_\theta\ell_o(a^{\mathrm{ali}};\theta)
  +\nabla_\theta\ell_o(a^{\mathrm{inc}};\theta).
\label{eq:app_partition_gradient}
\end{align}
These identities extend to batch averages under the same linear
reduction, without separately normalizing the two components.
This partition describes the dense coefficients; it is not the
R$^{2}$-OPD update rule.
At finite sharpness, R$^{2}$-OPD retains nonzero corrections in both
groups through positive reweighting; truncated responses retain standard OPD.

\subsection{Token-Level Alignment and Its Scope}
\label{app:alignment_geometry}

At $\bar\theta$, let
$\psi_t=\left.\nabla_\theta\log\pi_\theta(o_t\mid s_t)
\right|_{\theta=\bar\theta}$ and define
\begin{equation}
 d_t^{\mathrm{out}}=z\psi_t,
 \qquad d_t^{\mathrm{ali}}=a_t^{\mathrm{ali}}\psi_t,
 \qquad d_t^{\mathrm{inc}}=a_t^{\mathrm{inc}}\psi_t.
\end{equation}
The coefficient partition implies
\begin{align}
 \langle d_t^{\mathrm{ali}},d_t^{\mathrm{out}}\rangle
 &=m_t\mathbb I[zA_t>0]|A_t|\|\psi_t\|_2^2\geq0,
 \nonumber\\
 \langle d_t^{\mathrm{inc}},d_t^{\mathrm{out}}\rangle
 &=-m_t\mathbb I[zA_t<0]|A_t|\|\psi_t\|_2^2\leq0.
\label{eq:app_local_geometry}
\end{align}
The inequalities are strict on the corresponding active component
when $\psi_t\neq0$.
They characterize alignment with an outcome-signed direction for the
same token, not with the expected task gradient $g_R$.
Summing in the direct-sum space of token directions preserves these
signs; aggregation in the shared parameter space introduces cross-token
inner products and need not preserve them.
Consequently, token-level outcome alignment alone does not imply
reward ascent. The task-progress criterion in
Eq.~\eqref{eq:r2_task_gain} additionally depends on correction utility.

\paragraph{Trajectory centering.}
For $N_o=\sum_tm_t>0$, define
$\bar A_o=N_o^{-1}\sum_tm_tA_t$ and
$\widetilde A_t=A_t-\bar A_o$.
Then $\sum_tm_t\widetilde A_t=0$, so nonconstant valid advantages
acquire both positive and negative centered values.
Routing by $z\widetilde A_t$ therefore need not agree with routing
by $zA_t$; R$^{2}$-OPD uses the latter.
This response-level centering is also distinct from the conditional
action centering $A_t-\mathbb E_s[A_t]$ in
Eq.~\eqref{eq:mot_value_balance}.

\section{Teacher Preference and Student Continuation Value: Further Analysis}
\label{app:motivation-proofs}

This appendix proves Proposition~\ref{prop:mot_local_separation},
relates its covariance criterion to over-correction and over-trust,
and clarifies the outcome evidence motivating supervision allocation
in Section~\ref{sec:motivation}.

\subsection{Setup and Normalized Local Perturbation}
\label{app:motivation-setup}

Fix a prefix $s=s_t=(q,o_{<t})$ and a finite action set $\mathcal V_s$
on which $\pi_{\bar\theta}(\cdot\mid s)$ and $\pi_T(\cdot\mid s)$
are normalized and strictly positive.
A finite generation horizon and a common termination convention ensure
that the binary reward is defined, so $Q_R^{\bar\theta}(s,a)\in[0,1]$.
Only the next-action distribution at $s$ changes; the continuation
policy remains $\pi_{\bar\theta}$, leaving $Q_R^{\bar\theta}(s,a)$ fixed.
Unless stated otherwise, moments indexed by $s$ are taken over
$o_t\sim\pi_{\bar\theta}(\cdot\mid s)$, with $A_t$ defined in
Eq.~\eqref{eq:sampled_advantage}.

For $0\leq\epsilon\leq1$, the normalized interpolation in
Eq.~\eqref{eq:mot_local_interpolation} is
\begin{align}
 Z_s(\epsilon)
 &=\sum_{a\in\mathcal V_s}
 \pi_{\bar\theta}(a\mid s)^{1-\epsilon}\pi_T(a\mid s)^\epsilon,
 \label{eq:app_mot_normalizer}\\
 \pi_\epsilon(a\mid s)
 &=\frac{\pi_{\bar\theta}(a\mid s)^{1-\epsilon}
 \pi_T(a\mid s)^\epsilon}{Z_s(\epsilon)}.
 \label{eq:app_mot_path}
\end{align}
Thus $Z_s(0)=Z_s(1)=1$, $\pi_0=\pi_{\bar\theta}$, and $\pi_1=\pi_T$
at the fixed prefix. Finiteness and strict positivity imply smoothness
in a neighborhood of $\epsilon=0$.
Differentiating gives
\begin{equation}
 \frac{\partial}{\partial\epsilon}\log\pi_\epsilon(a\mid s)
 =\log\frac{\pi_T(a\mid s)}{\pi_{\bar\theta}(a\mid s)}
 -\frac{Z_s'(\epsilon)}{Z_s(\epsilon)}.
\label{eq:app_mot_log_derivative}
\end{equation}
At $\epsilon=0$,
\begin{equation}
 \frac{Z_s'(0)}{Z_s(0)}
 =\sum_a\pi_{\bar\theta}(a\mid s)
 \log\frac{\pi_T(a\mid s)}{\pi_{\bar\theta}(a\mid s)}
 =\mathbb E_s[A_t].
\label{eq:app_mot_mean}
\end{equation}
Writing $\dot\pi_\epsilon(a\mid s)=\partial_\epsilon\pi_\epsilon(a\mid s)$,
we obtain
\begin{equation}
 \dot\pi_0(a\mid s)
 =\pi_{\bar\theta}(a\mid s)
 \left(\log\frac{\pi_T(a\mid s)}{\pi_{\bar\theta}(a\mid s)}
 -\mathbb E_s[A_t]\right),
\label{eq:app_mot_probability_derivative}
\end{equation}
with $\sum_a\dot\pi_0(a\mid s)=0$.
This normalized probability-space intervention need not coincide
with gradient descent in shared model parameters.

\subsection{Proof of Proposition~\ref{prop:mot_local_separation}}
\label{app:motivation-local-proof}

\begin{proof}
The local reverse KL is
\[
 D_s(\epsilon)=\sum_a\pi_\epsilon(a\mid s)
 \log\frac{\pi_\epsilon(a\mid s)}{\pi_T(a\mid s)}.
\]
Differentiating and using $\sum_a\dot\pi_\epsilon(a\mid s)=0$ yields
\begin{align}
 D_s'(\epsilon)
 &=\sum_a\dot\pi_\epsilon(a\mid s)
 \log\frac{\pi_\epsilon(a\mid s)}{\pi_T(a\mid s)}
 +\sum_a\dot\pi_\epsilon(a\mid s)\nonumber\\
 &=\sum_a\dot\pi_\epsilon(a\mid s)
 \log\frac{\pi_\epsilon(a\mid s)}{\pi_T(a\mid s)}.
\label{eq:app_mot_kl_derivative_general}
\end{align}
At $\epsilon=0$, the log ratio is $-A_t$ for $o_t=a$.
Substituting Eq.~\eqref{eq:app_mot_probability_derivative} gives
\begin{align}
 D_s'(0)
 &=-\mathbb E_s\!\left[(A_t-\mathbb E_s[A_t])A_t\right]\nonumber\\
 &=-\operatorname{Var}_s[A_t],
\label{eq:app_mot_kl_result}
\end{align}
which proves Eq.~\eqref{eq:mot_kl_derivative}.

Since the continuation values are fixed, differentiating
$F_s(\epsilon)=\sum_a\pi_\epsilon(a\mid s)Q_R^{\bar\theta}(s,a)$
yields
\begin{align}
 F_s'(0)
 &=\sum_a\dot\pi_0(a\mid s)Q_R^{\bar\theta}(s,a)\nonumber\\
 &=\mathbb E_s\!\left[
 (A_t-\mathbb E_s[A_t])Q_R^{\bar\theta}(s,o_t)\right]\nonumber\\
 &=\operatorname{Cov}_s\!\left(A_t,Q_R^{\bar\theta}(s,o_t)\right),
\label{eq:app_mot_value_result}
\end{align}
proving Eq.~\eqref{eq:mot_value_derivative}.

Negative covariance implies that $A_t$ is nonconstant and hence
$\operatorname{Var}_s[A_t]>0$.
Smoothness gives
\begin{align}
 D_s(\epsilon)-D_s(0)
 &=-\epsilon\operatorname{Var}_s[A_t]+O(\epsilon^2),
 \label{eq:app_mot_kl_expansion}\\
 F_s(\epsilon)-F_s(0)
 &=\epsilon\operatorname{Cov}_s\!\left(
 A_t,Q_R^{\bar\theta}(s,o_t)\right)+O(\epsilon^2).
 \label{eq:app_mot_value_expansion}
\end{align}
Both first-order coefficients are strictly negative, so there exists
$\epsilon_0>0$ such that both quantities decrease for every
$0<\epsilon<\epsilon_0$.
\end{proof}

\subsection{Centered Preferences and the Two Failure Modes}
\label{app:motivation-failure-modes}

Since $F_s(0)=\mathbb E_s[Q_R^{\bar\theta}(s,o_t)]$ and
$\mathbb E_s[A_t-\mathbb E_s[A_t]]=0$, the reward derivative also admits
the centered form
\begin{equation}
 F_s'(0)=\mathbb E_s\!\left[
 (A_t-\mathbb E_s[A_t])
 (Q_R^{\bar\theta}(s,o_t)-F_s(0))\right].
\label{eq:app_mot_centered_balance}
\end{equation}
Under the local interpolation, the value-based forms of the two
failure modes are
\begin{align}
 \text{Over-correction:}\quad
 &A_t-\mathbb E_s[A_t]<0,\quad
 Q_R^{\bar\theta}(s,o_t)>F_s(0),
 \label{eq:app_mot_overcorrection}\\
 \text{Over-trust:}\quad
 &A_t-\mathbb E_s[A_t]>0,\quad
 Q_R^{\bar\theta}(s,o_t)<F_s(0).
 \label{eq:app_mot_overtrust}
\end{align}
The former shifts probability away from actions with above-average
student value; the latter shifts it toward actions with below-average value.
Both contribute negatively to Eq.~\eqref{eq:app_mot_centered_balance},
but the net derivative depends on the probability-weighted contributions
of all actions.

The distinction between raw and centered preference is essential.
The sign of $A_t$ describes an individual surrogate coefficient,
whereas normalized probability changes depend on $A_t-\mathbb E_s[A_t]$.
Because
\begin{equation}
 \mathbb E_s[A_t]=-K_t(\bar\theta)\leq0,
\label{eq:app_mot_preference_mean}
\end{equation}
a negative $A_t$ need not imply decreasing action probability along
the interpolation.
Neither sign alone determines an aggregate parameter update, which
also depends on the remaining token contributions and parameter sharing.

\subsection{What Trajectory Outcomes Reveal}
\label{app:motivation-outcome-evidence}

The signed outcome in Eq.~\eqref{eq:signed_outcome} provides evidence
about continuation quality. Conditional on the prefix and next action,
\begin{equation}
 \mathbb E\!\left[z(q,o)\mid s_t=s,o_t=a\right]
 =2Q_R^{\bar\theta}(s,a)-1,
\label{eq:app_mot_conditional_outcome}
\end{equation}
where the suffix follows $\pi_{\bar\theta}$.
For moments involving $z(q,o)$, the conditional rollout law includes
both the next action and the student-generated suffix.
Since $A_t$ depends only on the fixed prefix and next action,
the tower property gives
\begin{align}
 \operatorname{Cov}_s(A_t,z(q,o))
 &=\mathbb E_s\!\left[(A_t-\mathbb E_s[A_t])z(q,o)\right]\nonumber\\
 &=\mathbb E_s\!\left[(A_t-\mathbb E_s[A_t])
 \bigl(2Q_R^{\bar\theta}(s,o_t)-1\bigr)\right]\nonumber\\
 &=2\operatorname{Cov}_s\!\left(A_t,Q_R^{\bar\theta}(s,o_t)\right)
 =2F_s'(0).
\label{eq:app_mot_outcome_covariance}
\end{align}
This expectation-level identity links outcome evidence to the local
value criterion without identifying the utility of a single correction.
In particular, the sign of $z(q,o)A_t$ need not match the sign of an
action's contribution to Eq.~\eqref{eq:app_mot_centered_balance}:
$z$ compares a sampled reward with a fixed threshold, not with $F_s(0)$,
and contains continuation sampling noise.
Excluding outcomes from the definition of $A_t$ therefore does not imply
statistical independence between $A_t$ and $z(q,o)$.

R$^{2}$-OPD combines this outcome evidence with the sign and magnitude
of $A_t$ to reallocate teacher supervision.
Positive weights preserve coefficient signs and nonzero feedback,
while per-response normalization conserves absolute coefficient mass.
Proposition~\ref{prop:mot_local_separation} motivates this allocation
principle but does not identify optimal weights, establish the sigmoid
as task-optimal, or guarantee positive reward change under reweighting.
The method's allocation properties and empirical performance remain
distinct from this local diagnostic.

\subsection{An Analytic Two-Route Counterexample}
\label{app:mot-two-routes}

Consider a finite decision process in which two actions at a fixed
prefix select routes $a_1$ and $a_2$, followed by a terminal
success/failure action. The probabilities below specify an illustrative
construction, not empirical measurements from a language model.
\begin{center}
\begin{tabular}{@{}lccc@{}}
\toprule
Route & $\pi_{\bar\theta}(a\mid s)$ & $\pi_T(a\mid s)$
& $Q_R^{\bar\theta}(s,a)$ \\
\midrule
$a_1$ & $0.2$ & $0.8$ & $0.1$ \\
$a_2$ & $0.8$ & $0.2$ & $0.8$ \\
\bottomrule
\end{tabular}
\end{center}
The final column gives the student's continuation success probabilities.
Let the teacher succeed with probabilities $0.95$ and $0.90$ after
selecting $a_1$ and $a_2$, respectively.
Its total success probability is $0.8(0.95)+0.2(0.90)=0.94$,
whereas the student's is
\[
 F_s(0)=0.2(0.1)+0.8(0.8)=0.66.
\]

The distillation advantages are $\log4$ and $-\log4$, with mean
$-0.6\log4$ and centered values $1.6\log4$ and $-0.4\log4$.
The corresponding student values relative to $F_s(0)$ are $-0.56$
and $0.14$.
Thus, $a_1$ illustrates over-trust and $a_2$ over-correction in the
value-based sense above, with
\begin{align}
 D_s'(0)&=-0.64(\log4)^2<0,\\
 F_s'(0)&=0.2(0.8)(2\log4)(0.1-0.8)
 =-0.224\log4<0.
\end{align}
At $\epsilon=1$, route selection exactly matches the teacher,
so $D_s(1)=0$, compared with $D_s(0)=0.6\log4$.
Keeping the student's continuation policy fixed nevertheless gives
\[
 F_s(1)=0.8(0.1)+0.2(0.8)=0.24.
\]
The example isolates partial imitation: adopting the teacher's route
selection does not transfer its ability to complete those routes.
It establishes the possibility of the preference--value mismatch,
not its prevalence in language models or a guarantee that R$^{2}$-OPD
resolves this particular decision process.
\section{Core Proofs and Their Scope}
\label{app:r2_proofs}

For the coefficient-space results, fix a nonempty complete response $o$,
its finite advantages $A_t$, outcome $z=z(q,o)\in\{-1,+1\}$, and
scale $\nu_o>0$. These quantities remain fixed as $\beta$ varies.
Unless stated otherwise, $M_o=\sum_t|A_t|>0$ and $0\leq\beta<\infty$.
All sums are over valid response tokens, and normalization is exact.

\subsection{Proof of Proposition~\ref{prop:r2_decomposition}}
\label{app:r2_decomposition}

\begin{proof}
Substituting the sigmoid gate into Eq.~\eqref{eq:r2_mass_weight} gives
\begin{equation}
 A_t^{\mathrm{R2}}=Z_og_tA_t
 =Z_oA_t\sigma\!\left(\frac{\beta zA_t}{\nu_o}\right).
\label{eq:r2_coefficient_substitution}
\end{equation}
Using the identity
\begin{equation}
 \sigma(x)=\frac{e^x}{1+e^x}
 =\frac12\left(1+\tanh\frac{x}{2}\right),
\label{eq:r2_sigmoid_identity}
\end{equation}
we obtain
\begin{equation}
 A_t^{\mathrm{R2}}
 =\frac{Z_o}{2}A_t
 +\frac{Z_o}{2}A_t\tanh\!\left(\frac{\beta zA_t}{2\nu_o}\right).
\label{eq:r2_appendix_expansion}
\end{equation}
For $A_t\neq0$, oddness of $\tanh$ and
$A_t\operatorname{sgn}(A_t)=|A_t|$ imply
\begin{equation}
 A_t\tanh\!\left(\frac{\beta zA_t}{2\nu_o}\right)
 =z|A_t|\tanh\!\left(\frac{\beta|A_t|}{2\nu_o}\right).
\label{eq:r2_oddness}
\end{equation}
Both sides vanish when $A_t=0$.
Substitution into Eq.~\eqref{eq:r2_appendix_expansion} proves
Eq.~\eqref{eq:r2_decomposition}.
The identity holds pointwise in $\beta$, including the dependence
of $Z_o$ on $\beta$; no small-sharpness approximation is required.
\end{proof}

For nonzero $A_t$, the magnitude ratio of Term B to Term A is
$\tanh(\beta|A_t|/(2\nu_o))<1$, so the coefficient retains its sign.
This does not determine the token's probability change after a
shared-parameter batch update.

\subsection{Proof of Proposition~\ref{prop:r2_path}}
\label{app:r2_path}

\begin{proof}
Write $b_t=|A_t|$, $a_t=zA_t/\nu_o$, and $D_o=\sum_tg_tb_t$,
so $g_t=\sigma(\beta a_t)$ and $Z_o=M_o/D_o$.

\paragraph{Sign and mass preservation.}
For finite $\beta$, $0<g_t<1$ implies $D_o>0$ and $Z_o>0$.
Thus, multiplication by $Z_og_t$ preserves nonzero coefficients
and their signs, with
\begin{equation}
 \sum_t|A_t^{\mathrm{R2}}|
 =Z_o\sum_tg_tb_t=\frac{M_o}{D_o}D_o=M_o.
\label{eq:r2_appendix_mass}
\end{equation}

\paragraph{Monotone reallocation.}
Let $\mathcal A_o=\{t:a_t>0\}$ and $\mathcal C_o=\{t:a_t<0\}$;
zero-advantage positions carry no mass.
Define the unnormalized group masses
\begin{equation}
 B_{\mathrm A}=\sum_{t\in\mathcal A_o}b_tg_t,
 \qquad B_{\mathrm C}=\sum_{t\in\mathcal C_o}b_tg_t.
\label{eq:r2_gated_group_mass}
\end{equation}
The common normalizer cancels in the inconsistent mass share:
\begin{equation}
 \kappa_o(\beta)
 =\frac{Z_oB_{\mathrm C}}{Z_o(B_{\mathrm A}+B_{\mathrm C})}
 =\frac{B_{\mathrm C}}{B_{\mathrm A}+B_{\mathrm C}}.
\label{eq:r2_share_cancellation}
\end{equation}
With $a_t$ fixed,
\begin{equation}
 g_t'(\beta)=a_tg_t(1-g_t).
\label{eq:r2_gate_derivative}
\end{equation}
If both groups carry positive mass, then
$B_{\mathrm A},B_{\mathrm C}>0$,
$B_{\mathrm A}'>0$, and $B_{\mathrm C}'<0$.
The quotient rule therefore yields
\begin{equation}
 \kappa_o'(\beta)
 =\frac{B_{\mathrm C}'B_{\mathrm A}-B_{\mathrm C}B_{\mathrm A}'}
 {(B_{\mathrm A}+B_{\mathrm C})^2}<0.
\label{eq:r2_appendix_selectivity}
\end{equation}
At $\beta=0$, the derivative is right-sided.
Since total mass remains $M_o$, inconsistent mass $M_o\kappa_o$
decreases while aligned mass $M_o(1-\kappa_o)$ increases.

\paragraph{Endpoints.}
At $\beta=0$, $g_t=1/2$ and $Z_o=2$, giving
$A_t^{\mathrm{R2}}=A_t$.
As $\beta\to\infty$, $g_t$ tends to one on $\mathcal A_o$
and zero on $\mathcal C_o$.
If $M_o^{\mathrm{ali}}=\sum_{t\in\mathcal A_o}b_t>0$, then
$D_o\to M_o^{\mathrm{ali}}$ and $Z_o\to M_o/M_o^{\mathrm{ali}}$.
Hence
\begin{equation}
 \lim_{\beta\to\infty}A_t^{\mathrm{R2}}
 =\frac{M_o}{M_o^{\mathrm{ali}}}A_t\mathbb I[t\in\mathcal A_o].
\label{eq:r2_normalized_hard_endpoint}
\end{equation}
Zero advantages remain zero at both endpoints.
\end{proof}

\paragraph{Boundary cases and scope.}
If all nonzero advantages are aligned or all are inconsistent,
$\kappa_o$ is identically zero or one, respectively.
The all-inconsistent case is excluded from
Eq.~\eqref{eq:r2_normalized_hard_endpoint}:
for $z=1$ and $A=(-1,-1)$, identical gates give $Z_o=1/g$,
so the coefficients remain $(-1,-1)$ for every finite $\beta$.
For $M_o=0$, the coefficients are zero and $\kappa_o$ is undefined.
Monotonicity applies to group mass, not individual coefficients,
and the hard endpoint is not generally equivalent to unnormalized masking.
Recovering the full vanilla update at $\beta=0$ additionally requires
identical data, masks, loss reduction, optimizer state, and settings.

\subsection{Proof of Proposition~\ref{prop:r2_task_progress}}
\label{app:r2_task_progress}

Assume $J$ is differentiable at $\bar\theta$ and define
$g_R=\nabla J(\bar\theta)$ and
$\psi_t=\left.\nabla_\theta\log\pi_\theta(o_t\mid s_t)
\right|_{\theta=\bar\theta}$.
Rollout expectations are finite and admit differentiation under the
expectation, with the rollout distribution held fixed.
Differentiating the detached log-probability surrogate in
Eq.~\eqref{eq:r2_loss} gives
\begin{equation}
 -\left.\nabla_\theta\mathcal L_{\mathrm{R2}}(\theta;\bar\theta)
 \right|_{\theta=\bar\theta}
 =\mathbb E\!\left[\sum_t A_t^{\mathrm{R2}}\psi_t\right]
 =\mathbb E[d_{\mathrm{R2}}(o)].
\label{eq:r2_on_policy_gradient}
\end{equation}
Replacing $A_t^{\mathrm{R2}}$ by $A_t$ gives
$\mathbb E[d_{\mathrm V}(o)]$.

For comparison, the probability ratio
$\rho_t(\theta)=\pi_\theta(o_t\mid s_t)/\pi_{\bar\theta}(o_t\mid s_t)$
satisfies
\begin{equation}
 \nabla_\theta\rho_t(\theta)
 =\rho_t(\theta)\nabla_\theta\log\pi_\theta(o_t\mid s_t),
 \qquad \rho_t(\bar\theta)=1.
\label{eq:r2_ratio_derivative}
\end{equation}
Replacing the log probability by this ratio gives the same gradient
at $\bar\theta$, but not generally away from the rollout parameters.

\begin{proof}
For the fixed complete response, recall
$p_o(t)=|A_t|/M_o$ and
$u_t=\operatorname{sgn}(A_t)\langle g_R,\psi_t\rangle$.
Here $p_o$ is a distribution over token positions.
Its mean gate satisfies
\begin{equation}
 \mu_g:=\mathbb E_{p_o}[g]
 =\frac{\sum_t|A_t|g_t}{M_o}=\frac{1}{Z_o}>0.
\label{eq:r2_mean_gate}
\end{equation}
Using $A_t=|A_t|\operatorname{sgn}(A_t)$, the standard OPD projection is
\begin{equation}
 \langle g_R,d_{\mathrm V}(o)\rangle
 =\sum_t A_t\langle g_R,\psi_t\rangle
 =M_o\mathbb E_{p_o}[u].
\label{eq:r2_vanilla_projection}
\end{equation}
Similarly,
\begin{equation}
 \langle g_R,d_{\mathrm{R2}}(o)\rangle
 =Z_o\sum_tg_t|A_t|u_t
 =\frac{M_o\mathbb E_{p_o}[gu]}{\mu_g}.
\label{eq:r2_reweighted_projection}
\end{equation}
Subtracting yields
\begin{align}
 \langle g_R,d_{\mathrm{R2}}(o)-d_{\mathrm V}(o)\rangle
 &=\frac{M_o}{\mu_g}
 \left(\mathbb E_{p_o}[gu]-\mu_g\mathbb E_{p_o}[u]\right)
 \nonumber\\
 &=\frac{M_o\operatorname{Cov}_{p_o}(g,u)}{\mathbb E_{p_o}[g]},
\label{eq:r2_covariance_proof}
\end{align}
which proves Eq.~\eqref{eq:r2_task_gain}.
Zero-advantage positions have zero $p_o$ mass.
\end{proof}

\paragraph{A sufficient reward--utility agreement condition.}
\label{app:r2_agreement}
Suppose $\mathcal A_o$ and $\mathcal C_o$ both carry positive mass,
and let $\alpha=p_o(\mathcal A_o)\in(0,1)$.
For $G\in\{\mathcal A_o,\mathcal C_o\}$, define
$\mu_{g,G}=\mathbb E_{p_o}[g\mid G]$,
$\mu_{u,G}=\mathbb E_{p_o}[u\mid G]$, and
$c_G=\operatorname{Cov}_{p_o}(g,u\mid G)$.
Conditioning on the two groups gives
\begin{align*}
 \mathbb E_{p_o}[gu]
 &=\alpha(c_{\mathcal A}+\mu_{g,\mathcal A}\mu_{u,\mathcal A})
 +(1-\alpha)(c_{\mathcal C}+\mu_{g,\mathcal C}\mu_{u,\mathcal C}),\\
 \mathbb E_{p_o}[g]
 &=\alpha\mu_{g,\mathcal A}+(1-\alpha)\mu_{g,\mathcal C},\\
 \mathbb E_{p_o}[u]
 &=\alpha\mu_{u,\mathcal A}+(1-\alpha)\mu_{u,\mathcal C}.
\end{align*}
Subtracting the product of the last two expressions from the first yields
\begin{equation}
\begin{aligned}
 \operatorname{Cov}_{p_o}(g,u)
 &=\alpha c_{\mathcal A}+(1-\alpha)c_{\mathcal C}\\
 &\quad+\alpha(1-\alpha)
 (\mu_{g,\mathcal A}-\mu_{g,\mathcal C})
 (\mu_{u,\mathcal A}-\mu_{u,\mathcal C}).
\end{aligned}
\label{eq:r2_group_covariance}
\end{equation}
For $\beta>0$, the gate satisfies
$\mu_{g,\mathcal A}>1/2>\mu_{g,\mathcal C}$.
Consequently, $\mu_{u,\mathcal A}>\mu_{u,\mathcal C}$ and
$\alpha c_{\mathcal A}+(1-\alpha)c_{\mathcal C}\geq0$
suffice for a strictly positive per-response projection gain.
More generally, the between-group term must outweigh any negative
within-group contribution.
These are additional utility conditions, not consequences of
outcome alignment alone.

\subsection{Relative Progress and Reward Ascent}
\label{app:r2_finite_step}

Let $d_{\mathrm V}=\mathbb E[d_{\mathrm V}(o)]$ and
$d_{\mathrm{R2}}=\mathbb E[d_{\mathrm{R2}}(o)]$ use the same
rollout distribution and reduction, and define
\begin{equation}
 \Delta_R=\langle g_R,d_{\mathrm{R2}}-d_{\mathrm V}\rangle,
 \qquad P_{\mathrm V}=\langle g_R,d_{\mathrm V}\rangle.
\label{eq:r2_expected_task_gain}
\end{equation}
By linearity, $\Delta_R=\mathbb E[\Delta_o]$, where $\Delta_o$ is
the right-hand side of Eq.~\eqref{eq:r2_task_gain} for complete
positive-mass responses and zero for zero-mass or identity-fallback
responses.

Suppose $\nabla J$ is $L_J$-Lipschitz, with $L_J>0$, on a
neighborhood containing both update segments.
For either fixed direction $d$ and step size $\eta>0$,
\begin{equation}
 J(\bar\theta+\eta d)-J(\bar\theta)-\eta\langle g_R,d\rangle
 =\int_0^\eta\langle\nabla J(\bar\theta+td)-g_R,d\rangle\,dt.
\label{eq:r2_taylor_integral}
\end{equation}
Cauchy--Schwarz and Lipschitz continuity bound the absolute remainder
by $L_J\eta^2\|d\|_2^2/2$.
Applying the lower bound to $d_{\mathrm{R2}}$ and the upper bound
to $d_{\mathrm V}$ gives
\begin{equation}
\begin{aligned}
 &J(\bar\theta+\eta d_{\mathrm{R2}})-J(\bar\theta+\eta d_{\mathrm V})\\
 &\quad\geq\eta\Delta_R-\frac{L_J\eta^2}{2}
 \left(\|d_{\mathrm{R2}}\|_2^2+\|d_{\mathrm V}\|_2^2\right).
\end{aligned}
\label{eq:r2_relative_reward_bound}
\end{equation}
If $\Delta_R>0$, relative reward improvement follows whenever both
update segments remain in the stated neighborhood and
\[
 0<\eta<\frac{2\Delta_R}
 {L_J\left(\|d_{\mathrm{R2}}\|_2^2+\|d_{\mathrm V}\|_2^2\right)}.
\]

Reward ascent from the original policy is a separate condition.
Since $\langle g_R,d_{\mathrm{R2}}\rangle=P_{\mathrm V}+\Delta_R$,
\begin{equation}
 J(\bar\theta+\eta d_{\mathrm{R2}})-J(\bar\theta)
 \geq\eta(P_{\mathrm V}+\Delta_R)
 -\frac{L_J\eta^2}{2}\|d_{\mathrm{R2}}\|_2^2.
\label{eq:r2_absolute_reward_bound}
\end{equation}
If $P_{\mathrm V}+\Delta_R>0$, this guarantees local ascent for
\[
 0<\eta<\frac{2(P_{\mathrm V}+\Delta_R)}
 {L_J\|d_{\mathrm{R2}}\|_2^2}.
\]
In particular, overcoming a negative baseline projection
$P_{\mathrm V}<0$ at first order requires $\Delta_R>-P_{\mathrm V}$,
not merely $\Delta_R>0$.
These bounds concern the specified deterministic directions,
not arbitrary stochastic or adaptive-optimizer updates.

\subsection{Gradient Compatibility and Task Progress}
\label{app:r2_compatibility}

On a fixed batch of complete responses, define the aligned OPD loss
by restricting Eq.~\eqref{eq:opd_surrogate} to $z(q,o)A_t>0$.
Let $G_{\mathrm{ali}}=\nabla\mathcal L_{\mathrm{ali}}(\bar\theta)$ and
$G_{\mathrm{R2}}=\nabla\mathcal L_{\mathrm{R2}}(\bar\theta)$ use the
same reduction.
If $\mathcal L_{\mathrm{ali}}$ is $L_{\mathrm{ali}}$-smooth along
the update segment, then
\begin{equation}
\begin{aligned}
 &\mathcal L_{\mathrm{ali}}(\bar\theta-\eta G_{\mathrm{R2}})
 -\mathcal L_{\mathrm{ali}}(\bar\theta)\\
 &\quad\leq-\eta\langle G_{\mathrm{ali}},G_{\mathrm{R2}}\rangle
 +\frac{L_{\mathrm{ali}}\eta^2}{2}\|G_{\mathrm{R2}}\|_2^2.
\end{aligned}
\label{eq:r2_actual_compatibility}
\end{equation}
A positive inner product therefore permits aligned-surrogate descent
for a sufficiently small step, but does not establish reward ascent.
Likewise, coefficient-mass preservation does not fix update magnitude:
\begin{equation}
 \|d_{\mathrm{R2}}(o)\|_2^2
 =\sum_{i,j}A_i^{\mathrm{R2}}A_j^{\mathrm{R2}}
 \langle\psi_i,\psi_j\rangle.
\label{eq:r2_gram_norm}
\end{equation}
The cross-token inner products determine how gradient contributions
reinforce or cancel one another.

\section{Implementation Conventions}
\label{app:r2_implementation}

\paragraph{Fallback and masking.}
\label{app:r2_truncation}
Responses terminated by the length limit retain
$A_t^{\mathrm{R2}}=A_t$; their outcomes are not used for reweighting.
This fallback preserves standard OPD supervision rather than discarding
responses.
Prompt and padding positions are excluded from the scale, mass, and loss;
empty responses contribute no token loss.
Complete zero-mass responses retain zero coefficients with $Z_o=1$.

\paragraph{Relation to the analyzed surrogate.}
\label{app:r2_policy_loss}

For a fixed batch, all coefficients, masks, response lengths, and snapshot
log probabilities are detached.
Under consistent trainer evaluations, $\rho_{i,t}(\bar\theta)=1$
and $\left.\nabla_\theta\rho_{i,t}(\theta)\right|_{\bar\theta}
=\psi_{i,t}$.
Since the ratio lies inside the clipping interval,
Eq.~\eqref{eq:impl_actual_policy_loss} has the rollout-point gradient
\begin{equation}
 -\left.\nabla_\theta\widehat{\mathcal L}_{\mathrm{impl}}
 \right|_{\bar\theta}
 =\frac{1}{N}\sum_{i=1}^{N}\frac{1}{\widetilde n_i}
   \sum_t m_{i,t}A_{i,t}^{\mathrm{R2}}\psi_{i,t}.
\label{eq:impl_rollout_gradient_bridge}
\end{equation}
This equals the gradient direction of the detached log-probability
surrogate with the same per-response reduction.
The equivalence is exact under the stated forward-evaluation assumption;
separate floating-point passes may introduce numerical deviations.

For $X\in\{\mathrm V,\mathrm{R2}\}$, let
$\widehat d_X^{\mathrm{resp}}=N^{-1}\sum_i
 d_X(o_i)/\widetilde n_i$, with the per-response directions defined in
Section~\ref{subsec:r2_task_progress} and zero direction for empty responses.
Applying Eq.~\eqref{eq:r2_task_gain} response by response gives
\begin{equation}
 \left\langle g_R,
   \widehat d_{\mathrm{R2}}^{\mathrm{resp}}
   -\widehat d_{\mathrm V}^{\mathrm{resp}}\right\rangle
 =\frac{1}{N}\sum_{i\in\mathcal I}
   \frac{M_{o_i}}{n_i}
   \frac{\operatorname{Cov}_{p_{o_i}}(g,u)}
        {\mathbb E_{p_{o_i}}[g]},
\label{eq:impl_response_weighted_task_gain}
\end{equation}
where $\mathcal I$ indexes complete, positive-mass responses;
identity-fallback responses have zero difference.
Thus, response averaging preserves the within-response allocation
identity while introducing inverse-length weights in batch aggregation.
Population directions accordingly use
$\mathbb E[d_X(o)/\widetilde n_o]$ rather than $\mathbb E[d_X(o)]$.

The reverse-KL gradient identity in
Eq.~\eqref{eq:opd_gradient_consistency} concerns the token-summed
analysis surrogate. It does not automatically extend to inverse-length
weighting, since realized response length depends on sampled actions.
The coefficient identities and the fixed-batch gradient relation above
remain valid independently of that expectation-level identity.
These relations describe the policy-loss gradient before global norm
clipping and AdamW; finite-step guarantees for direct gradient updates
are not guarantees for arbitrary adaptive-optimizer displacements.

\paragraph{Exact normalization.}
\label{app:r2_numerics}
The scale floor in $\nu_o$ preserves the coefficient identities,
whereas adding a positive constant to the mass denominator does not.
For $M_o>0$, $D_o=\sum_tg_t|A_t|$, and
$Z_o=M_o/(D_o+\epsilon_{\mathrm{den}})$ with $\epsilon_{\mathrm{den}}>0$,
\[
 \sum_t|A_t^{\mathrm{R2}}|
 =\frac{M_o}{D_o+\epsilon_{\mathrm{den}}}\sum_tg_t|A_t|
 =\frac{M_oD_o}{D_o+\epsilon_{\mathrm{den}}}<M_o.
\]
This modification breaks exact mass conservation and the vanilla
endpoint at $\beta=0$.
Strict gate positivity holds in exact arithmetic; floating-point
saturation and underflow require separate numerical care.

\end{document}